\documentclass{article}

\usepackage[final]{neurips_2022}
\usepackage{enumitem}
\usepackage{mystyle}
\usepackage{hyperref}
\usepackage{bbold}
\usepackage{wrapfig}
\usepackage{amsmath}
\usepackage{graphicx}
\usepackage{amsmath}

\usepackage{todonotes}
\newcommand{\dd}{{\mathrm d}} % upshape derivative

\newcommand\reb{\textcolor{black}}

\usepackage[utf8]{inputenc} % allow utf-8 input
\usepackage[T1]{fontenc}    % use 8-bit T1 fonts
\usepackage{hyperref}       % hyperlinks
\usepackage{url}            % simple URL typesetting
\usepackage{booktabs}       % professional-quality tables
\usepackage{amsfonts}       % blackboard math symbols
\usepackage{nicefrac}       % compact symbols for 1/2, etc.
\usepackage{microtype}      % microtypography
\usepackage{xcolor}         % colors

\title{Do Residual Neural Networks discretize Neural Ordinary Differential Equations?}

\author{
  Michael E. Sander \\
ENS, CNRS\\
 Paris, France\\
 \texttt{michael.sander@ens.fr} \\
  \And
  Pierre Ablin \\
  Université Paris-Dauphine, CNRS \\
  Paris, France \\
  \texttt{pierreablin@gmail.com} \\
  \And
  Gabriel Peyré \\
  ENS, CNRS \\
  Paris, France \\
  \texttt{gabriel.peyre@ens.fr} \\
}

\begin{document}

\maketitle

\begin{abstract}
Neural Ordinary Differential Equations (Neural ODEs) are the continuous analog of Residual Neural Networks (ResNets). We investigate whether the discrete dynamics defined by a ResNet are close to the continuous one of a Neural ODE. We first quantify the distance between the ResNet's hidden state trajectory and the solution of its corresponding Neural ODE. Our bound is tight and, on the negative side, does not go to $0$ with depth $N$ if the residual functions are not smooth with depth. On the positive side, we show that this smoothness is preserved by gradient descent for a ResNet with linear residual functions and small enough initial loss. It ensures an implicit regularization towards a limit Neural ODE at rate $\frac1N$, uniformly with depth and optimization time. As a byproduct of our analysis, we consider the use of a memory-free discrete adjoint method to train a ResNet by recovering the activations on the fly through a backward pass of the network, and show that this method theoretically succeeds at large depth if the residual functions are Lipschitz with the input. We then show that Heun's method, a second order ODE integration scheme, allows for better gradient estimation with the adjoint method when the residual functions are smooth with depth. We experimentally validate that our adjoint method succeeds at large depth, and that Heun’s method needs fewer layers to succeed. We finally use the adjoint method successfully for fine-tuning very deep ResNets without memory consumption in the residual layers.
\end{abstract}

\section{Introduction}
\paragraph{Problem setup.} Residual Neural Networks (ResNets) \citep{he2016deep, he2016identity} keep on outperforming state of the art in computer vision \citep{wightman2021resnet, bello2021revisiting}, and more generally skip connections are widely used in a various range of applications \citep{vaswani2017attention, dosovitskiy2020image}. 
A ResNet  of depth $N$ iterates, starting from $x_0 \in \RR^d$, $x_{n+1} = x_n + f(x_n,\theta_n^N)$ and outputs a final value $x_N \in \RR^d$ where $f$ is a neural network called residual function. In this work, we consider a simple modification of this forward rule by letting explicitly the residual mapping depend on the depth of the network:
\begin{equation}\label{eq:ResNet}
    x_{n+1} = x_n + \frac{1}{N}f(x_n,\theta_n^N).
\end{equation}
On the other hand, a Neural ODE \citep{chen2018Neural} uses a neural network $\phi_{\Theta}(x, s)$, that takes time $s$ into account, to parameterise a vector field \citep{kidger2022Neural} in a differential equation, as follows,
\begin{equation}\label{eq:Neural_ode}
  \frac{\dd x }{\dd s} = \phi_{\Theta}(x(s),s)
   \quad\text{with}\quad
  x(0) = x_0, 
\end{equation}
and outputs a final value $x(1)\in \RR^d$, the solution of Eq.\eqref{eq:Neural_ode}.
The Neural ODE framework enables learning without storing activations (the $x_n$'s) using the adjoint state method, hence significantly reducing the memory usage for backpropagation that can be a bottleneck during training \citep{wang2018superneurons,peng2017large,zhu2017unpaired, gomez2017reversible}.
\begin{wrapfigure}{r}{0.45\textwidth}
\label{fig:intro}
\includegraphics[width=0.45\textwidth]{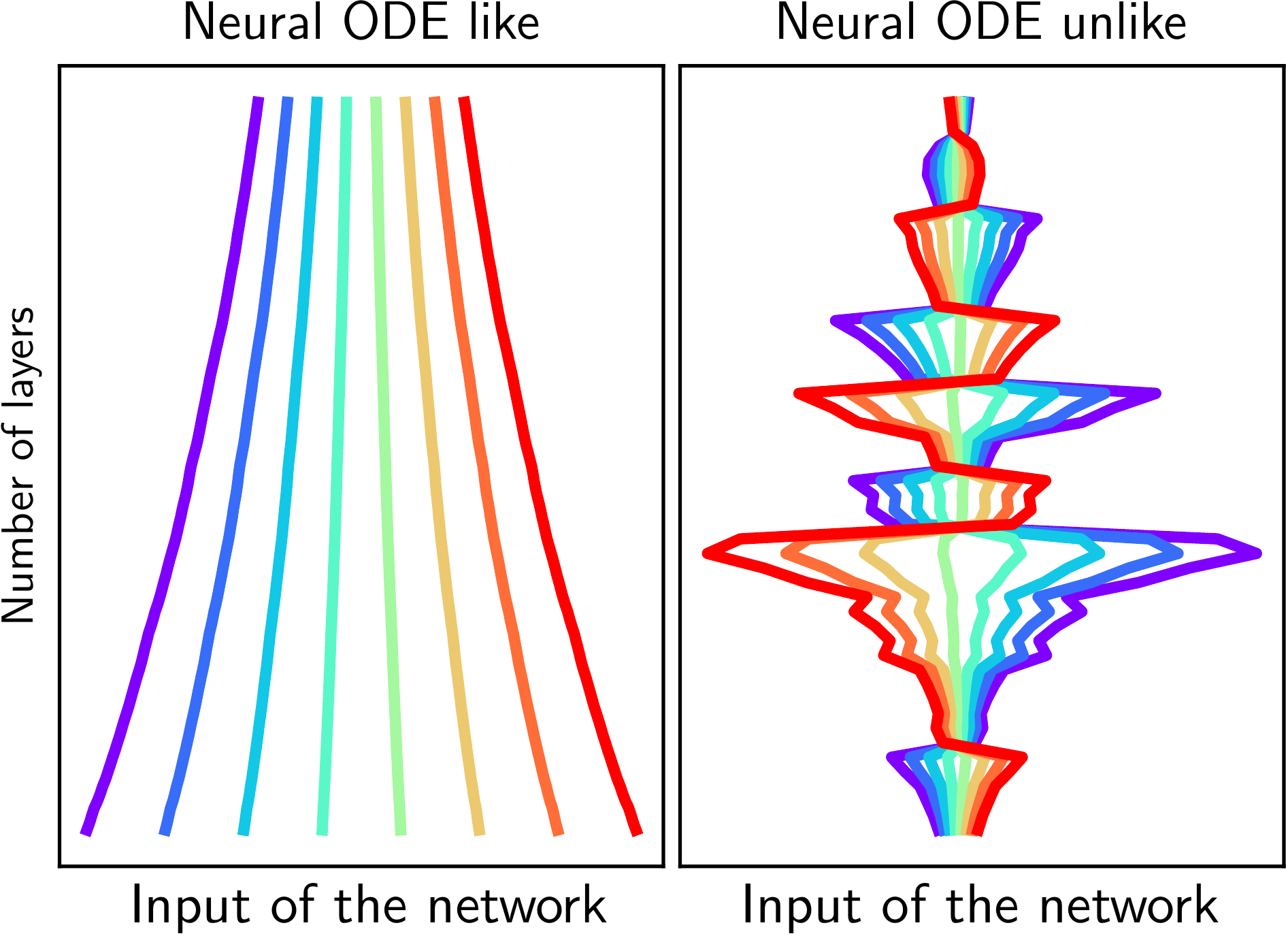}
\caption{\small{\textbf{Trajectory of ResNets with 300 layers.} Left: we learn $x \to \frac {x^2}{2}$, trajectories are smooth and do not intersect. Right: we learn $x \to \frac {-x^2}{2}$, trajectories are not smooth and intersect.}}
\vspace{-10pt}
\end{wrapfigure}
Neural ODEs also provide a theoretical framework to study deep learning models from the continuous viewpoint, using the arsenal of ODE theory \citep{dupont2019augmented, li2019deep, teshima2020universal}.
Importantly, they can also be seen as the continuous analog of ResNets.
Indeed, consider for $N$ an integer, the Euler scheme for solving Eq.~\eqref{eq:Neural_ode} with time step $\frac1N$ starting from $x_0$ and iterating $x_{n+1} = x_n + \frac1N \phi_{\Theta}(x_n,\frac nN)$. Under mild assumptions on $\phi_{\Theta}$, this scheme is known to converge to the true solution of Eq.~\eqref{eq:Neural_ode} as $N$ goes to $+\infty$. Also, if $\Theta = (\theta_n^N)_{i\in[N-1]}$ and $ 
\phi_{\Theta}(.,\frac nN) = f(., \theta_n^N)$, then the ResNet equation Eq.~\eqref{eq:ResNet} corresponds to a Euler discretization with time step $\frac1N$ of Eq.~\eqref{eq:Neural_ode}. However, for a given ResNet  with fixed depth $N$ and weights, the activations in Eq.~\eqref{eq:ResNet} can be far from the solution of Eq.~\eqref{eq:Neural_ode}. 
This is illustrated in Figure \ref{fig:intro} where we show that a deep ResNet can easily break the topology of the input space, which is impossible for a Neural ODE. 
In this paper, we study the link between ResNets and Neural ODEs.
%We are interested in how well a ResNet discretizes a Neural ODE, \todo{Gab: pas clair} if the closeness to a Neural ODE is kept during the optimization process, and finally on the applicability of the adjoint method for training ResNets. 
We make the following contributions:
\begin{itemize}[leftmargin=*]
    \item In Section \ref{sec:resnet_discretize}, we propose a framework to define a set of associated Neural ODEs for a given ResNet. We control the error between the discrete and the continuous trajectory. We show that without additional assumptions on the smoothness with depth of the residual functions, this error does not go to $0$ as $N\to\infty$ (Prop. \ref{prop:approx_error}). However, we show that under some assumptions on the weight initialization, the trained parameters of a deep linear ResNet uniformly (with respect to both depth and training time) approach a Lipschitz function as the depth $N$ of the network goes to infinity, at speed $\frac1N$ (Prop. \ref{prop:smoothness} and Th. \ref{th:limit_map}). This result highlights an implicit regularization towards a limit Neural ODE.
    %We illustrate this theorem on synthetic data. 
    \item In Section \ref{sec:adjoint}, we investigate a simple technique to train ResNets without storing activations. Inspired by the adjoint method, we propose to recover the approximated activations during the backward pass by using a reverse-time Euler scheme. We control the error for recovering the activations and gradients with this method. We show that if the residuals of the ResNet are bounded and Lipschitz continuous, with constants independent of $N$, then this error scales in $O(\frac1N)$ (Prop. \ref{prop:reconstruction_error}).
    Hence, the adjoint method needs a large number of layers to lead to correct gradients (Prop. \ref{prop:gradient_error}). We then consider a smoothness-dependent reconstruction with Heun’s method to bound the error between the true and approximated gradient by a term that depends on $\frac1N$ times the smoothness in depth of the residual functions, hence guaranteeing a better approximation when successive weights are close one to another (Prop. \ref{prop:reconstruction_error_heun} and \ref{prop:gradient_error_heun}).
    \item In Section \ref{sec:experiments}, on the experimental side, we show that the adjoint method fails when training a ResNet 101 on ImageNet. Nevertheless, we empirically show that very deep ResNets pretrained with tied weights \reb{(constant weights: $\theta^N_n = \theta$ $\forall n$)} can be refined -using our adjoint method- on CIFAR-10 and ImageNet by untying their weights, leading to a better test accuracy. Last, but not least, we show using a ResNet architecture with heavy downsampling in the first layer that our adjoint method succeeds at large depth and that Heun's method leads to a better behaved training, hence confirming our theoretical results.
    %We also show that it corresponds to a framework for which ResNets actually discretize a Neural ODE. \pierre{maybe add a more accurate statement?}.
\end{itemize}
\section{Background and related work}
\paragraph{Neural ODEs.} Neural ODEs are a class of implicit deep learning models defined by an ODE where a neural network parameterises the vector field \citep{E_2017,chen2018Neural,dupont2019augmented,sun2018stochastic,E_2018,lu2017finite,ruthotto2018deep, kidger2022Neural}. Given an input $x_0$, the output of the model is the solution of the ODE \eqref{eq:Neural_ode} at time $1$. From a theoretical viewpoint, the expression capabilities of Neural ODEs have been investigated in \citep{cuchiero2020deep, teshima2020universal, li2019deep} and the Neural ODE framework has been used to better understand the dynamics of more general architectures that include residual connections such as Transformers \citep{sander2022sinkformers, lu2019understanding}.
Experimentaly, Neural ODEs have been successful in a various range of applications, among which physical modelling \citep{greydanus2019hamiltonian, cranmer2020lagrangian} and generative modeling \citep{chen2018Neural,grathwohl2018ffjord}. However, there are many areas where Neural ODEs have failed to replace ResNets, for instance for building computer vision classification models. Neural ODEs fail to compete with ResNets on ImageNet, and to the best of our knowledge, previous works using Neural ODEs on ImageNet consider weight-tied architectures and only achieves the same accuracy as a ResNet18 \citep{zhuang2021mali}. 
\paragraph{Implicit Regularization of ResNets towards ODEs.}
Recent works have studied the link between ResNets and Neural ODEs. In \citep{cohen2021scaling}, the authors carry experiments to better understand the
scaling behavior of weights in ResNets as a function of the depth. They show that under the assumption that there exists a scaling limit $\theta(s) = N^{\beta} \lim \theta^N_{\lfloor Ns \rfloor }$ for the weights of the ResNets (with $0< \beta < 1$) and if the scale of the ResNet is $\frac1{N^{\alpha}}$ with $0 < \alpha < 1 $ and $\alpha + \beta = 1$, then the hidden state of the ResNet converges to a solution of a linear ODE. In this paper, we are interested in the case where $\alpha=1$, which seems more natural since it is the scaling that appears in Euler's method with step $\frac1N$.
%Though this case is also englobed by \citet{cohen2021scaling} if the activation function is the identity, we 
 In addition, we do not assume the existence of a scaling limit $\theta(s) = \lim \theta^N_{\lfloor Ns \rfloor }$. In subsection \ref{subsec:linear}, we demonstrate the existence of this scaling limit in the linear setting, under some assumptions. The recent work \citep{cont2022convergence} shows results regarding linear convergence of gradient descent in ResNets and prove the existence of an $\frac12$-Hölder continuous scaling limit as $N \to \infty$ with a scaling factor for the residuals in $\frac{1}{\sqrt{N}}$ which is different from ours. In contrast, we show that our limit function is Lipschitz continuous, which is a stronger regularity. We also show that our convergence is uniform in depth and optimization time. More generally, recent works have proved the convergence of gradient descent training of ResNet when the initial loss is small enough. This include ResNet with finite width but arbitrary large depth~\citep{du2019gradient,liu2020linearity} and ResNet with both infinite width and depth
\citep{lu2020mean,barboni2021global}. These convergence proofs leverage an implicit bias toward weights with small amplitudes. They however leave open the question of convergence of individual weights as depth increases, which we tackle in this work in the linear case. This requires showing an extra bias toward weights with small variations across depth. 

\paragraph{\reb{Memory bottleneck in ResNets.}}
\reb{Training deep learning models involve graphics processing units (GPUs) where memory is a practical bottleneck \citep{wang2018superneurons,peng2017large,zhu2017unpaired}. Indeed, backpropagation requires to store activations at each layer during the forward pass. Since samples are processed using mini batches, this storage can be important. For instance, with batches of size 128, the memory needed to compute gradients for a ResNet 152 on ImageNet is about 22 GiB. Note that the memory needed to store the parameters of the model is only 220 MiB, which is negligible compared to the memory needed to store the activations.
Thus, designing deep invertible architectures where one can recover the activations on the fly during the backpropagation iterations has been an active field in recent years \citep{gomez2017reversible, sander2021momentum, jacobsen2018irevnet}. In this work, we propose to approximate activations using a reverse-time Euler scheme, as we detail in the next subsection.} 
\paragraph{Adjoint Method.} Consider a loss function $L(x_N)$ for the ResNet~\eqref{eq:ResNet}. The backpropagation  equations~\citep{baydin2015automatic} are
\begin{equation}\label{eq:backprop}
    \nabla_{\theta^N_{n-1}}L =  \frac{1}{N}[\partial_\theta f(x_{n-1}, \theta^N_{n-1})]^\top \nabla_{x_n}L,\quad
\nabla_{x_{n-1}}L = [I + \frac1N  \partial_xf(x_{n-1}, \theta^N_{n-1})]^\top \nabla_{x_n}L.
\end{equation}
Now, consider a loss function $L(x(1))$ for the Neural ODE~\eqref{eq:Neural_ode}. The adjoint state method \citep{pontryagin1987mathematical, chen2018Neural} gives 
\begin{equation}\label{eq:adjoint_state}
    \nabla_{\Theta} L = - \int_{T}^{0}  {\partial_{\Theta} [\phi_{\Theta}(x(s), s)}]^\top \nabla_{x(s)} L \dd s, \quad -\dot{\nabla}_{x(s)} L(s) = [\partial_x {\phi_{\Theta}(x(s), s)}]^\top \nabla_{x(s)} L.
\end{equation}
%and $a(t) = \nabla_{x(t)} L$. 
Note that if $\Theta = (\theta_n^N)_{n\in[N-1]}$ and $ 
\phi_{\Theta}(.,\frac nN) = f(., \theta_n^N)$, then Eq.~\eqref{eq:backprop} corresponds to a Euler discretization with time step $\frac1N$ of Eq.~\eqref{eq:adjoint_state}. The key advantage of using  Eq.~\eqref{eq:adjoint_state} is that one can recover $x(s)$ on the fly by solving the Neural ODE~\eqref{eq:Neural_ode} backward in time starting from $x(1)$. This strategy avoids storing the forward trajectory $(x(s))_{s\in[0, 1]}$ and leads to a $O(1)$ memory footprint \citep{chen2018Neural}. In this work, we propose to use a discrete adjoint method by using a reverse-time Euler scheme for approximately recovering the activations in a ResNet (Section \ref{sec:adjoint}).  Contrarily to other models such as RevNets \citep{gomez2017reversible} (architecture change) or Momentum ResNets \citep{pmlr-v139-sander21a} (forward rule modification) which rely on an exactly invertible forward rule, the proposed method requires no change at all in the network, but gives approximate gradients.
%This will naturally leads to some reconstructions errors, presented in Section \ref{sec:adjoint}. 
\paragraph{\reb{Notations.}} \reb{For $k \in \NN$, $\Cc^k$ is the set of functions $f : \RR^d \to \RR^d$ $k$ times differentiable with continuous $k^{th}$ differential. If $f \in \Cc^1$, $\partial_x f(x)[y]$ is the differential of $f$ at $x$ evaluated in $y$. For $K \subset \RR^d$ compact, $\|.\|$ a norm and $f$ a continuous function on $\RR^d$, we denote $\|f\|^K_{\infty} = \sup_{x\in K} \| f (x) \|$.}
 
\section{ResNets as discretization of Neural ODEs}\label{sec:resnet_discretize}
In this section we first show that without further assumptions, the distance between the discrete trajectory and the solution of associated ODEs can be constant with respect to the depth of the network if the residual functions lack smoothness with depth. We then present a positive result by studying the linear case where we show that, under some hypothesis (small loss initialization and initial smoothness with depth), the ResNet converges to a Neural ODE as the number of layers goes to infinity. We show that this convergence is uniform with depth and optimization time.
%\pierre{it would be nice to have a small paragraph about the two times that are involved in the paper (optimization and depth)}.
\subsection{Distance to an ODE}
We first define associated Neural ODEs for a given ResNet.
\begin{defn}
We say that a neural network $\phi_{\Theta} : \RR^d \times \RR \to \RR^d$ smoothly interpolates the ResNet Eq.~\eqref{eq:ResNet}.
%using the Neural networks $f$ parametrized by $(\theta_0, .., \theta_{N-1})$ 
if $\phi_{\Theta}$ is smooth and
%if there exists $0 \leq t_0 < t_1 < .. < t_{N-1} \leq 1$ such that 
$\forall n \in \{0, ..., N-1\}$, $\phi_{\Theta}(., \frac nN) = f(., \theta_n^N)$.
\end{defn}
Note that we omit the dependency of $\Theta$ in $N$ to simplify notations.
For example, for a given ResNet, there are two natural ways to interpolate it with a Neural ODE, either by interpolating the \emph{residuals}, or by interpolating the \emph{weights}. Indeed, one can interpolate the residuals with $\phi_{\Theta}(\cdot, s) =  (n+1- Ns)f(., \theta_n^N) + (Ns - n) f(\cdot, \theta_{n+1}^N)$ when $s\in [\frac n N, \frac{n+1}N]$, or interpolate the weights with   $\phi_{\Theta}(\cdot, s) = f(\cdot, (n+1 - Ns)\theta_n^N + (Ns - n)\theta_{n+1}^N)$ for $s\in [\frac n N, \frac{n+1}N]$. If $\theta_n^N = \theta^N$ does not depend on $n$, then both interpolations are identical and one can simply consider  $\phi_{\Theta}(x, s) = f(x, \theta^N)$, $\forall (x, s)$.

We now consider \reb{any} smooth interpolation $\phi_{\Theta}$ for the ResNet \eqref{eq:ResNet} and a Euler scheme for the Neural ODE \eqref{eq:Neural_ode} with time step $\frac{1}{N}$.
\begin{prop}[Approximation error]\label{prop:approx_error}
 We suppose that $\phi_{\Theta}$ is $\Cc^1$, and $L$-Lipschitz with respect to $x$, uniformly in $s$. Note that this implies that the solution of Eq.~\eqref{eq:Neural_ode} is well defined, unique, $\Cc^2$, and that the trajectory is included in some compact $K \subset \RR^d$. Denote $C_N \eqdef \| \partial_s\phi_{\Theta} + \partial_x\phi_{\Theta} [\phi_{\Theta}]\|^{K \times [0, 1]}_{ \infty}$. Then one has for all $n$:
$
\| x_n - x(\frac{n}{N})\| \leq \frac{e^{L}-1}{2NL}C_N
$
\reb{if $L > 0 $ and
$\| x_n - x(\frac{n}{N})\| \leq \frac{C_N}{2N}$ if $L = 0$.}

%\gab{Peut être noter $C_N \triangleq \| \partial_t \phi_{\Theta} + \partial_x\phi_{\Theta} [\phi_{\Theta}]\|^{K \times [0, 1]}_{ \infty}$, dire que la borne est $\leq C_N/N$ et ensuite bien expliquer que la question c'est savoir si $C_N/N \rightarrow 0$ et que c'est pas clair en general. }
\end{prop}
%\begin{sproof}  We adapt a proof from \citep{demailly2016analyse}, we define $\varepsilon_n = x(\frac{n+1}{N}) - x(\frac{n}{N}) - h\phi_{\Theta}(x(\frac{n}{N}), \frac{n}{N})$. We show that  $\|\varepsilon_n\| \leq \frac{1}{2N^2} \| \ddot{x} \|_{\infty}$. The true error we are interested in is $e_n = x(\frac{n}{N}) - x_n$. We use $\phi_{\Theta}$'s Lipschitz property to show that $\|e_{n+1}\| \leq (1 + \frac LN)\|e_n\| +  \|\varepsilon_n\|.$ Finally, the discrete Gronwall's lemma implies the desired result. \end{sproof}
For a full proof, see appendix \ref{proof:approx_error}. Note that this result extends Theorem 3.2 from \citep{zhuang2020adaptive} to the non-autonomous case: our bound depends on $\partial_s\phi_{\Theta}.$ 
%\reb{This quantity is also well defined at $L=0$ since $\frac{e^L - 1}{L} \to 1$ as $L \to 0$.}
Finally, our bound is tight. Indeed, for $\phi_{\Theta}(x, s) = as + b$ for $a, b \in \RR^d$, we get $\partial_s \phi_{\Theta} + \partial_x\phi_{\Theta} [\phi_{\Theta}]= a$, $L=0$ and $\|x(1) - x_N\| = \frac {\|a\|} {2N}$.
%\pierre{need a stronger case for the tightness, for instance what if L is given?}
\paragraph{Implication.} The tightness of our bound shows that closeness to the ODE solution is not guaranteed, because we do not know whether $C_N/N \rightarrow 0$. Indeed, consider first the residual interpolation $\phi_{\Theta}(x, s) =  \left( (n+1 - Ns)f(x, \theta_n^N) + (Ns - n) f(x, \theta_{n+1}^N)\right)) \mathbb{1}_{s \in [\frac nN, \frac {n+1}N]}$ and the simple case where $\partial_x\phi_{\Theta} [\phi_{\Theta}] = 0$. We get $\partial_s \phi_\Theta(x,s) = N (f(x,\theta_{n+1}^N) - f(x, \theta_{n}^N))\mathbb{1}_{s \in [\frac nN, \frac {n+1}N]}$, which corresponds to the discrete derivative. It means that although there is a $\frac1N$ factor in our bound, the time derivative term -- without further regularity with depth of the weights, which is at the heart of subection~\ref{subsec:linear} -- usually scales with $N$: $\frac{C_N}N=O(1)$.
%\gab{en reutilisant la notation que je suggere, on pourrait dire $C_N/N=O(1)$ mais a priori ne converge pas vers 0}.
As a first example, consider the simple case where $f(x, \theta_n^N) = n.$ This gives $x_N = x_0 + \frac{(N-1)}{2}$ while the integration of the Neural ODE \eqref{eq:Neural_ode} leads to $x(1) =x_0 +  \frac N2$ because $\phi_{\Theta}(x, s) = Ns$, so the $\|x_N - x(1)\| = \frac{1}{2}$ is not small. Intuitively, this shows that weights cannot scale with depth when using the residual interpolation.
Now, consider the weight interpolation, $\theta_n^N = (-1)^n$ and suppose $f$ is written as $f(x, \theta) = \theta^2$. This gives $\phi_{\Theta}(., s) = (2Ns -(2n+1))^2$ when $s\in [\frac n N, \frac{n+1}N]$. Integrating, we get $x(1) = \frac{1}{3}$ while the output of the ResNet is $x_N = 1$. Hence $\|x_N - x(1)\| = \frac{2}{3}$ is also not small, even though the weights are bounded.
%unless the Neural ode is smooth enough
%This gives  $\| x_n - x(\frac{n}{N})\| \leq \frac{e^{L}-1}{2L} \sup_{n\in [N-1]} \| f(.,\theta_{n+1}^N) - f(.,\theta_{n}^N)\|_{\infty}^{K},$ which is tight. However, it is not guaranteed that  $\| f(.,\theta_{n+1}^N) - f(.,\theta_{n}^N)\|_{\infty}^{K} =o(1)$: one needs additional regularity assumptions on the weights of the ResNet to obtain a Neural ODE in the large depth limit. 
Thus, one needs additional regularity assumptions on the weights of the ResNet to obtain a Neural ODE in the large depth limit. Intuitively if the weights are initialized close from one another and they are updated using gradient descent, they should  stay close from one another during training, 
%\gab{il faudrait expliquer pourquoi les gradient ont un peu de regularite en profondeur, c'est pas evident je trouve} 
since the gradients in two consecutive layers will be similar, as highlighted in Eq. \eqref{eq:backprop}.
Indeed, we see that if $x_{n}$ and $x_{n+1}$ are close, then  $\nabla_{x_n}L$ and $\nabla_{x_{n+1}} L$ are close, and then if $\theta_n^N$ and $\theta_{n+1}^N$ are close, $\nabla_{\theta_n^N}L$ and $\nabla_{\theta_{n+1}^N}L$ are also close. In subsection \ref{subsec:linear}, we formalize this intuition and present a positive result for ResNets with linear residual functions. More precisely, we show that with proper initialization, the difference between two successive parameters is in $\frac1N$ during the entire training. Furthermore, we show that the weights of the network converge to a smooth function, hence defining a limit Neural ODE. 

\subsection{Linear Case}\label{subsec:linear}
As a further step towards a theoretical understanding of the connections between ResNets and Neural ODEs we investigate the linear setting, where the residual functions are written $f(x, \theta) = \theta x$ for any $\theta \in \RR^{d\times d}$. It corresponds to a deep matrix factorization problem \citep{zou2020global, bartlett2018gradient, arora2019implicit, arora2018optimization}. As opposed to these previous works, we study the infinite depth limit of these linear ResNets with a focus on the learned weights. We show that, if the weights are initialized close one to another, then at any training time, the weights stay close one to another (Prop. \ref{prop:smoothness}) and importantly, they converge to a smooth function of the continuous depth $s$ as $N \to \infty$ (Th. \ref{th:limit_map}). All the proofs are available in appendix \ref{app:proofs}.
\paragraph{Setting.}
 Given a training set $(x_k, y_k)_{k\in[n]}$ in $\RR^d$, we solve the regression problem of mapping $x_k$ to $y_k$ with a linear ResNet, \emph{i.e.} $f(x, \theta) = \theta x$, of depth $N$ and parameters $(\theta_1^N, \dots, \theta_N^N)$. The ResNet therefore maps $x_k$ to $\Pi^N x_k$ where $\Pi^N \eqdef \prod_{n=1}^N (I_d + \frac{\theta_n^N}{N}) = (I_d + \frac{\theta_N^N}{N}) \cdots (I_d + \frac{\theta_1^N}{N})$. It is trained by minimizing the average errors $\|\Pi^N x_k - y_k\|_2^2$, which is equivalent to the deep matrix factorization problem:
\begin{equation}\label{eq:min_problem}
    \argmin_{(\theta_n^N)_{i\in[N-1]}}L(\theta_1^N, \dots, \theta_N^N) \eqdef  \|\Pi^N -B\|_{\Sigma}^2,
\end{equation} where $\|A\|^2_{\Sigma} = \mathrm{Tr}(A\Sigma A^T)$, 
%\pierre{why not $C$ instead of $\Sigma$? Cause i use C as a constant at some point}
$\Sigma$ is the empirical covariance matrix of the data: $\Sigma \eqdef \frac1n \sum_{k=1}^nx_kx_k^\top$, and $B \eqdef \frac1n \sum_{k=1}^ny_kx_k^\top\Sigma^{-1}$. As is standard, we suppose that $\Sigma$ is non degenerated. We denote by $M>0$ (resp. $m >0$) its largest (resp. smallest) eigenvalue. 
\paragraph{Gradient.} We denote $\Pi^N_{:n} \eqdef (I_d + \frac{\theta_N^N}{N})\cdots (I_d + \frac{\theta_{n+1}^N}{N})$ and  $\Pi^N_{n:} \eqdef (I_d + \frac{\theta_{n-1}^N}{N})\cdots (I_d + \frac{\theta_{1}^N}{N})$
and write the gradient \reb{ $\nabla^N_n(t)= \nabla_{\theta_n^N}L(\theta^N_1(t), .. \theta^N_n(t), .., \theta^N_N(t))$}. The chain rule gives $N \nabla^N_n =\Pi_{:n}^{N\top}(\Pi^N - B)\Sigma\Pi_{n:}^{N\top} $. 
%The gradient of the cost function with respect to $\theta_n^N$ is $\nabla_{\theta_n^N} = \frac1N \left[(I +\frac{\theta_1^N}N)\dots (I+\frac{\theta_{n-1}^N}N)\right]^\top(\prod_{i=1}^N (I_d + \frac{\theta_n^N}{N}) - B)\left[(I +\frac{\theta_{n+1}^N}N)\dots (I+\frac{\theta_{N}^N}N)\right]^\top$. 
Intuitively, as $N$ goes to $+\infty$, the products $\Pi^N$, $\Pi_{:n}^{N}$ and $\Pi_{n:}^{N}$ should converge to some limit, hence we see that $N\nabla_n^N$ scale as $1$. Therefore, we train $\theta_n^N$ by the rescaled gradient flow $\frac{d\theta_n^N}{dt}(t) = - N\nabla^N_n(t)$ to minimize $L$ and denote  ${\ell}^N(t) = L(\theta_1^N(t), \dots, \theta_N^N(t))$.
%\pierre{ need to specify the O in appendix}
\paragraph{Two continuous variables involved.} Our results involve two continuous variables: $s \in [0, 1]$ is the depth of the limit network and corresponds to the time variable in the Neural ODE, whereas $t\in \RR_{+}$ is the gradient flow time variable.
As is standard in the analysis of convergence of gradient descent for linear networks, we consider the following assumption:
\begin{asp}\label{asp:init}
Suppose that at initialisation one has $\sqrt{{\ell}^N(0)} < \frac{m}{4\sqrt{2Me^3}}$ and  $\| \theta_n^{N}(0)\| \leq \frac14$. 
\end{asp}
Assumption \ref{asp:init} is the classical assumption in the literature \citep{zou2020global, barboni2021global} to prove linear convergence of our loss and that the $\theta^N_n(t)$'s stay bounded with $t$. \reb{Note that this bounded norm assumption implies that $\frac1N \theta_n(0) = O(\frac1N)$. This is in contrast with classical initialization scales in the \textit{feedforward} case where the initialization only depends on width \citep{he2015delving}. However this initialization scale is coherent with those of \textit{ResNets} for which the scale has to depend on depth \citep{yang2017mean}. In addition, the experimental findings in \cite{cohen2021scaling} suggest that the weights in ResNets scale in $\frac1{N^\beta}$ with $\beta > 0$.} 

We now prove an implicit regularization result showing that if at initialization, in addition to assumption \ref{asp:init}, the weights are close from one another ($O(\frac1N)$), they will stay at distance $O(\frac1N)$: the discrete derivative stay in $O(\frac1N)$, which is a central result to consider the infinite depth limit in our Th. \ref{th:limit_map}.
\begin{prop}[Smoothness in depth of the weights]\label{prop:smoothness}
Suppose assumption \ref{asp:init}. Suppose that there exists $C_0 >0$ independent of $n$ and $N$ such that $\|\theta_{n+1}^N(0)- \theta_n^N(0)\| \leq \frac{C_0}{N}$. Then, 
$\forall t \in \RR_+$, $\|\theta_n^N(t)\| < \frac{1}{2}$, and  $\theta_n^N(t)$ admits a limit $\psi_n^{N}$ as ${t \to +\infty}$. Moreover, there exists $C > 0$ such that $\forall t \in \RR_+$, $\|\theta_{n+1}^N(t)- \theta_n^N(t)\| \leq \frac{C}{N}$.
\end{prop}
%\begin{sproof}
%\vspace{-1.2em}
%In~\cite{zou2020global}, it is shown that $\|\theta_n^N(t)\|\leq \frac12$ and that the loss satisfies a Polyak-Łojasiewicz (PL) condition. Then we use the PL condition to show that $\nabla^N_n$ is integrable so that $\theta_n^N(t)$ admits a limit as $t \to \infty$. We then remark that $\nabla^N_{n+1} - \nabla^N_n = (I + \frac{\theta^{N\top}_{n+1}}{N})^{-1}(\frac{\nabla^N_n\theta_n^{N\top} - \theta_{n+1}^{N\top} \nabla_n^N}{N})$. This implies that  $\|\nabla^N_{n+1} - \nabla^N_n \| \leq \frac2N\|\nabla^N_n\|$. The integration of this relation along with the PL conditions leads the existence of $C >0$  independent of $n$, $N$ and $t$ such that $\|\theta_{n+1}^N(t)- \theta_n^N(t)\| \leq \frac{C}{N}.$
%\end{sproof}
For a full proof, see appendix \ref{proof:smoothness}. The inequality $\|\theta_{n+1}^N(t)- \theta_n^N(t)\| \leq \frac{C}{N}$ corresponds to a discrete Lipschitz property in depth. Indeed, for $s \in [0,1]$ and $t \in \RR_+$, let $\psi_N(s, t) = \theta^N_{\lfloor Ns \rfloor }(t)$. Then our result gives $\|\psi_N(\frac{n+1}N, t) - \psi_N(\frac{n}N, t)\| \leq \frac CN$ which implies that $\|\psi_N(s_1, t) - \psi_N(s_2, t)\| \leq C|s_1 - s_2| + \frac CN$.
 We now turn to the infinite depth limit $N \to \infty$.
 Th. \ref{th:limit_map} shows that there exists a limit function $\psi$ such that $\psi_N$ converges uniformly to $\psi$ in depth $s$ and optimization time $t$. Furthermore, this limit is Lipschitz continuous in $(s,t)$.
 In addition, we show that the ResNet $\Pi^N$ converges to the limit Neural ODE defined by $\psi$ that is preserved along the optimization flow, exhibiting an implicit regularization property of deep linear ResNets towards Neural ODEs.

\begin{thm}[Existence of a limit map]\label{th:limit_map}
Suppose assumption \ref{asp:init}, $\|\theta_{n+1}^N(0)- \theta_n^N(0)\| \leq \frac{C_0}{N}$ for some $C_0 > 0$ and that there exists a function $\psi_{\mathrm{init}} : [0, 1] \to \RR^{d\times d}$ such that $\psi_{N}(s, 0) \to \psi_{\mathrm{init}}(s)$ in $\|.\|_{\infty}$ uniformly in $s$ as $N \to \infty$, at speed $\frac1N$.
%(e.g. $\theta_n^{N}(0) = 0_{d\times d}$). %
Then the sequence $(\psi_{N})_{N \in \NN}$ uniformly converges (in $\|.\|_{\infty}$ w.r.t $(s,t)$) to a limit $\psi$ Lipschitz continuous in $(s, t)$ and $\|\psi - \psi_N\|_{\infty}= O(\frac1N)$. Furthermore, $\Pi^N$ uniformly converges as $N \to \infty$ to the mapping $x_0 \to x_1$ where $x_1$ is the solution at time $1$ of the Neural ODE $\frac{\dd x}{\dd s}(s) = \psi(s,t) x(s)$ with initial condition $x_0$.
\end{thm}
\begin{wrapfigure}{r}{0.45\textwidth}
\includegraphics[width=0.45\textwidth]{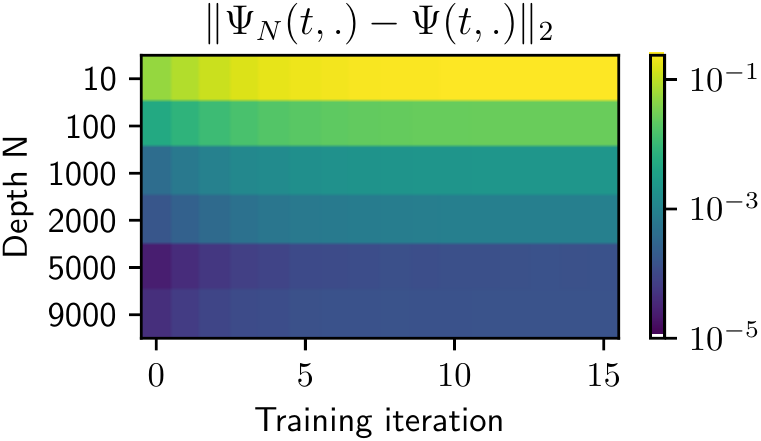}
\caption{\small{
%\textbf{Evolution} in depth $s \in [0,1]$ of some coefficients of the matrices $\theta_n^N(t)$ for $n \in [N-1]$ at random optimization steps $t \in \RR_+$.
\reb{$L_2$ norm $\|\Psi_N(t, .)  - \Psi(t, .)\|_2$ (w.r.t depth $s$) for different training iterations $t$ (horizontal axis) and different depth $N$ (vertical axis). As predicted by Th. \ref{th:limit_map}, this distance goes to $0$ as $N \to +\infty$.}
}}\label{fig:linear_weights}
\vspace{-25pt}
\end{wrapfigure}
We illustrate Th. \ref{th:limit_map} in Figure \ref{fig:linear_weights}. The assumption on the existence of $\psi_{\mathrm{init}}$ ensures a convergence at speed $\frac1N$ to a Neural ODE at optimization time $0$. Note that for instance, the constant initialization $\theta_n^{N}(0) = \theta_0 \in \RR^{d\times d}$ satisfies this hypothesis. 
 In order to prove Th. \ref{th:limit_map}, for which a full proof is presented in appendix \ref{proof:limit_map}, we first present a useful lemma: the weights of the network have at least one accumulation point. 
\begin{lem}[Existence of limit functions]\label{lemm:limit_functions}
For $s \in [0,1]$ and $t \in \RR_+$, let $\psi_N(s, t) = \theta^N_{\lfloor Ns \rfloor }(t)$. Under the assumptions of Prop. \ref{prop:smoothness}, there exists a subsequence $\psi_{\sigma(N)}$ and $\psi_{\sigma} : [0,1] \times \RR_+ \to \RR^{d\times d}$ Lipschitz continuous with respect to both parameters $s$ and $t$ such that $\psi_{\sigma(N)} \to \psi_{\sigma}$ uniformly (in $\|.\|_{\infty}$ w.r.t $(s,t)$). 
\end{lem}
%\vspace{-1em}
%begin{sproof}
%We first prove that $\Psi_N$ satisfies the almost equicontinuity property (i) $\| \psi_N(s_1) - \psi_N(s_2) \| \leq C'|t_1-t_2| + C|s_1-s_2| + \frac{C}{N}$ and is uniformly bounded (ii) $\forall N\in \NN$, $\|\psi_N \|_{\infty} \leq \frac12$. We then adapt a proof of the Arzelà–Ascoli theorem to prove the existence of a uniformly convergent subsequence and show its Lipschitz property.
%\end{sproof}
%vspace{-1em}
Lemma \ref{lemm:limit_functions} is proved in appendix \ref{proof:limit_functions}, and gives us the existence of a Lipschitz continuous accumulation point, but not the uniqueness nor the convergence speed.
For the uniqueness, we show in appendix \ref{proof:limit_map} that, under the assumptions of Th. \ref{th:limit_map}, one has that any accumulation point of $\psi_{\sigma}$ satisfies the limit Neural ODE
\begin{equation*}
    \partial_t \psi_{\sigma}(., t) = F(\psi_{\sigma}(., t)), \quad \psi_{\sigma}(., 0) = \psi_{\mathrm{init}}(., 0),
\end{equation*}
and show that $F$ satisfies the hypothesis of the Picard–Lindelöf theorem, hence showing the uniqueness of $\psi$.
We finally show that, as intuitively expected, trajectories of the weights of our linear ResNets of depth $N$ and $2N$ remain close one to each other. This gives the convergence speed in Th.~\ref{th:limit_map}. See appendix \ref{proof:closeness} for a proof.

\begin{lem}[Closeness of trajectories]\label{lemm:closeness}
Suppose asumption \ref{asp:init},  $\|\theta_{n+1}^N(0)- \theta_n^N(0)\| \leq \frac{C_0}N$ for some $C_0 >0 $ and that $\|\theta_{n}^N(0)- \theta_{2n}^{2N}(0)\| = O(\frac{1}{N})$. 
Then $\forall t \in \RR_+$, $\|\theta_n^N(t)- {\theta}_{2n}^{2N}(t)\| =O(\frac1N)$.
\end{lem}
%\begin{sproof}
%\mic{todo}
%\end{sproof}

%\pierre{need explanations about the result}
% and prove that under some assumptions the trained weights 
\section{Adjoint Method in Residual Networks}\label{sec:adjoint}
In this section, we focus on a particularly useful feature of Neural ODEs and its applicability to ResNets: their memory free backpropagation thanks to the adjoint method.
%All proofs are available in appendix \ref{app:proofs}. 
We consider a ResNet \eqref{eq:ResNet} and try to invert it using reverse mode Euler discretization of the Neural ODE \eqref{eq:Neural_ode} when $\phi_{\Theta}$ is any smooth interpolation of the ResNet.
%Without any architecture change or forward rule modification, one way to recover $x(t)$ is through a reverse mode Euler scheme for the Neural ODE \eqref{eq:Neural_ode}.
This corresponds to defining $\tilde{x}_{N} = x_N$ and iterate for $n \in \{N-1, \dots 0\}$: 
\begin{equation}\label{eq:backward_pass}
\tilde{x}_{n} = \tilde{x}_{n+1} - \frac1N f(\tilde{x}_{n + 1}, \theta^N_{n}).
\end{equation}
%This will naturally lead to reconstruction errors and errors in gradients, illustrated in section \ref{sec:adjoint}.
%\pierre{Here we need to be clearer about what we are trying to do from the get-go: say that we will consider not a Neural ode but a resnet and try to invert it using backwards euler discretization.} 
We then use the approximated activations $(\tilde{x}_n)_{n\in[N]}$ as a proxy for the true activations $(x_n)_{n\in[N]}$ to compute gradients without storing the activations:
\begin{equation}\label{eq:backprop_proxy}
    \tilde{\nabla}_{\theta^N_{n-1}}L =  \frac{1}{N}[\partial_\theta f(\tilde{x}_{n-1}, \theta^N_{n-1})]^\top \nabla_{\tilde{x}_n}L,\quad
\nabla_{\tilde{x}_{n-1}}L = [I + \frac1N  \partial_xf(\tilde{x}_{n-1}, \theta^N_{n-1})]^\top \nabla_{\tilde{x}_n}L.
\end{equation}
The approximate recovery of the activations in Eq.~\eqref{eq:backward_pass} is implementable for \emph{any} ResNet: there is no need for particular architecture or forward rule modification. The drawback is that the recovery is only approximate. We devote the remainder of the section to the study of the corresponding errors and to error reduction using second order Heun's method.
We first show that, if $f(., \theta^N_n)$ and its derivative are bounded by a constant independent of $N$, then the error for reconstructing the activations in the backward scheme \eqref{eq:backward_pass} is $O(\frac1N)$. Proofs of the theoretical results are in appendix \ref{app:proofs}.
%Second, we show that, as a consequence, the error in the gradients is in $O(\frac{1}{N^2})$. %Importantly, and in contrast with Prop. \ref{prop:approx_error}, the discrete derivative $f(x,\theta_{n+1}^N) - f(x, \theta_{n}^N)$ is not involved in the expression derived\pierre{need better explanation}.
\paragraph{Error for reconstructing activations.}
%We now show a similar result as in Prop. \ref{prop:approx_error} on the error made on the reconstruction of the activations $\Tilde{x}_n$ when using the reverse mode Euler scheme in the adjoint method. 
We consider the following assumption:
\begin{asp}\label{asp:bound}
There exists constants $C_f$ and $L_{f}$ such that $\forall N \in \NN$, $\forall n \in [N-1]$, $\|f(., \theta_n^N)\|_{\infty} \leq C_f$ and $\|\reb{\partial_x}[f(., \theta^N_n)]\|_{\infty} \leq L_{f}$. 
\end{asp}
Then the error made by reconstructing the activations is in $O(\frac{1}{N})$.
\begin{prop}[Reconstruction error]\label{prop:reconstruction_error}
With assumption \ref{asp:bound}, one has
$
\| x_n - \Tilde{x}_{n}\| \leq \frac{(e^{L_f} - 1)C_f }{N} + O(\frac{1}{N^2}).
$
\end{prop}
Prop. \ref{prop:reconstruction_error} shows a slow convergence of the error for recovering activations.
%Under assumption \ref{asp:bound}, our result from Prop. \ref{prop:approx_error} reads, for $\phi_{\Theta}$ the interpolation between the activation: $\|x_n - x(\frac nN)\|\leq \frac{(e^{L_f} - 1)C_f }{2N} + \frac12\|f(.,\theta_{n+1}^N) - f(.,\theta_{n}^N)\|$.
%The first term corresponds to the magnitude and variation of $f$ in space, whereas the second term corresponds to the smoothness in depth of $f$.
This bound does not depend on the discrete derivative $ f(.,\theta_{n+1}^N) - f(.,\theta_{n}^N)$, contrarily to the errors between the ResNet activations and the trajectory of the interpolating Neural ODE in Prop~\ref{prop:approx_error}. In summary, even though regularity in depth is necessary to imply closeness to a Neural ODE, it is not necessary to recover activations, and neither gradients, as  we now show. 
\paragraph{Error in gradients when using the adjoint method.}
We use the result obtained in Prop. \ref{prop:reconstruction_error} to derive a bound  in $O(\frac{1}{N^2})$ on the error made for computing gradients using formulas \eqref{eq:backprop_proxy}.
\begin{prop}[Gradient error]\label{prop:gradient_error}
Suppose assumption \ref{asp:bound}. Suppose in addition that  $\partial_xf(., \theta)$ admits a Lipschitz constant $L_{df}$, $\partial_\theta f(., \theta)$ admits a Lipschitz constant $\Delta$, and an upper bound $\Omega$, all of which are independent of $\theta$. Then one has $\| \tilde{\nabla}_{\theta^N_n}L - \nabla_{\theta^N_n}L\| = O(\frac{1}{N^2}).$
%Then there exists a constant $C_L$ such that $\| \nabla_{x_n}L \| \leq C_L$, 
%$
%\| \nabla_{x_n}L - \nabla_{\tilde{x}_n}L\| \leq \frac{(e^{L_f} - 1)(e^{L_f}-1) C_f %L_{df}C_L}{L_fN} + O(\frac{1}{N^2})
%$
%and 
%$$
%\| \tilde{\nabla}_{\theta^N_n}L - \nabla_{\theta^N_n}L\| \leq %\frac{e^{L_f}C_fL_{f}C_L(L_{\theta} + e^{L_{f}}{C_{\theta}}L_{df})}{N^2} + %O(\frac{1}{N^3}).
%$$
\end{prop}
For a proof, see appendix \ref{proof:gradient_error}, where we give the dependency of our upper bound as a function of $\Delta, L_f, C_f, \Omega$ and $L_{df}$.

\paragraph{Smoothness-dependent reconstruction with Heun's method.}
The bounds in Prop. \ref{prop:reconstruction_error} and \ref{prop:gradient_error} do not depend on the smoothness with respect to the weights of the $f(., \theta_n^N)$. Only the magnitude of the residuals plays a role in the correct recovery of the activations and estimation of the gradient. Hence, there is no apparent benefit of having such a network behave like a Neural ODE. We now turn to Heun's method, a second order integration scheme, and show that in this case smoothness in depth of the network improves activation recovery. A HeunNet~\citep{maleki2021heunnet} of depth $N$ with parameters $\theta_1^N, \dots, \theta_N^N$ iterates for $n=0,\dots, N-1$:
\begin{equation}
\label{eq:heun_fwd}
    y_{n} = x_n + \frac1N f(x_n,\theta_n^N)\quad \text{and} \quad x_{n+1} = x_n +\frac1{2N}(f(x_n, \theta_n^N) + f(y_n, \theta_{n + 1}^N)).
\end{equation}
These forward iterations can once again be approximately reversed by doing for $n=N-1,\dots,0$:
\begin{equation}
\label{eq:heun_bwd}
    \tilde{y}_{n} = \tilde{x}_{n+1} - \frac1N f(\tilde{x}_{n+1},\theta_{n + 1}^N)\quad \text{and} \quad \tilde{x}_{n} = \tilde{x}_{n+1} -\frac1{2N}(f(\tilde{x}_{n + 1}, \theta_{n+1}^N) + f(\tilde{y}_n, \theta_{n}^N)),
\end{equation}
which also enables approximated backpropagation without storing activations. When discretizing an ODE, Heun's method has a better $O(\frac1{N^2})$ error, hence we expect a better recovery than in Prop. \ref{prop:reconstruction_error}. Indeed, we have:
\begin{prop}[Reconstruction error - Heun's method]\label{prop:reconstruction_error_heun}
Assume assumption \ref{asp:bound}. Denote by $L_f'$ the Lipschitz constant of $x \mapsto \frac12(f(x,\theta_{n+1}^N) + f(x - \frac1Nf(x, \theta^N_{n+1}) ,\theta_{n}^N))$, by $L_\theta$ the Lispchitz constant of $\theta \mapsto f(\cdot, \theta)$ and by $L'_{\theta}$ that of $\theta \mapsto \partial_xf(., \theta)$. Let $C'_f= \frac14L_\theta'L_{\theta}$. Finally, define $\Delta_\theta^N \eqdef \max_{n}\|\theta_{n+1}^N -\theta_n^N\|^2 $.
Using Heun's method, we have:
$
\|x_n - \tilde{x}_n\| \leq \frac{(e^{L_f'} - 1)C'_f}{L_f'N}\times\Delta_{\theta}^N  + O(\frac1{N^2}).
$
\end{prop}
This bound is very similar to that in proposition~\ref{prop:reconstruction_error}, with an additional factor $\Delta_\theta^N$.
Hence, we see that under the condition that  $\Delta_\theta^N = O(\frac1N)$, the reconstruction error  $\|x_n -\tilde{x}_n\|$ is in $O(\frac1{N^2})$. In the linear case, we have proven under some hypothesis in Prop.~\ref{prop:smoothness} that such a condition on $\Delta_\theta^N$ holds during training. Consequently, the smoothness of the weights of a HeunNet in turns helps it recover the activations, while it is not true for a ResNet.
This provides better guarantees on the error on gradients:
\begin{prop}[Gradient error - Heun's method]\label{prop:gradient_error_heun}
Suppose assumption \ref{asp:bound}. Suppose in addition that  $\partial_xf(., \theta)$ admits Lipschitz constant, $\partial_\theta f(., \theta)$ admits a Lipschitz constant and an upper bound, all of which are independent of $\theta$. Then one has $\| \tilde{\nabla}_{\theta^N_n}L - \nabla_{\theta^N_n}L\| = O(\frac{\Delta_\theta^N}{N^2} + \frac1{N^3}).$
\end{prop}
Just like with activation, we see that Heun's method allows for a better gradient estimation when the weights are smooth with depth. Equivalently, for a fixed depth, this proposition indicates that HeunNets have a better estimation of the gradient with the adjoint method than ResNets which ultimately leads to better training and overall better performances by such memory-free model.
\section{Experiments}\label{sec:experiments}
We now present experiments to investigate the applicability of the results presented in this paper. We use Pytorch \citep{paszke2017automatic} and Nvidia Tesla V100 GPUs. Our code will be open-sourced. All the experimental details are given in appendix \ref{app:exp_details}, and we provide a recap on ResNet architectures in appendix \ref{app:resnet_recap}.
\begin{figure}[h]
\centering
\includegraphics[width=1\columnwidth]{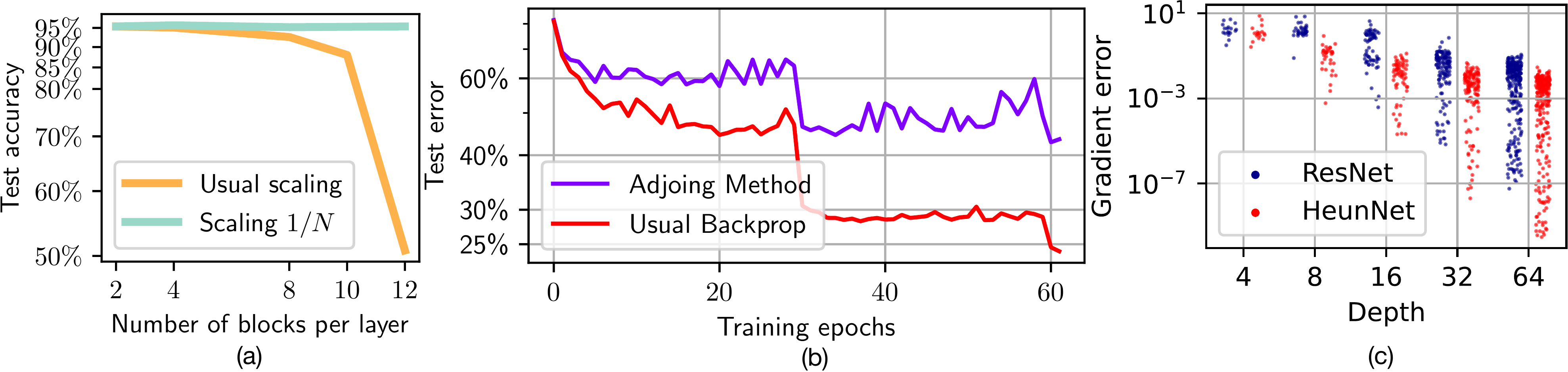} 
\caption{\small{(a) \textbf{Test accuracy on CIFAR-10} as a function of the number of blocks in each layer of the ResNet. Within each layer, weights are tied (\reb{3 runs}). (b) \textbf{Failure} of the adjoint method with a ResNet-101 on ImageNet (the approximated gradients are only used in the third layer of the network, that contains $23$ blocks). (c) \textbf{Relative error} between the approximated gradients using adjoint method and the true gradient, whether using a ResNet or a HeunNet. Each point corresponds to one parameter.}}\label{fig:adjoint}
\end{figure}
\subsection{Validation of our model with step size $\frac1N$}
The ResNet model \eqref{eq:ResNet} is different from the classical ResNet because of the $\frac1N$ term. This makes the model depth aware, and we want to study the impact of this modification on the accuracy on CIFAR and ImageNet.
\begin{wraptable}[5]{r}{0.4\linewidth}
	\vspace{-6mm}
	\caption{Test accuracy (ResNet-101)}
	\label{tab:perf_101}
	\centering
	\scalebox{0.9}{
		\begin{tabular}{lcccc}
			\toprule
			\multicolumn{1}{l}{} & ResNet-101 & Ours \\
			\midrule
			\multicolumn{1}{l}{CIFAR-10} & $95.5 \pm 0.1$\% & $95.5 \pm 0.1$\%  \\
			\midrule
			\multicolumn{1}{l}{ImageNet} & 77.8\% & 77.9\%  \\
			% \bottomrule
		\end{tabular}
	}
\end{wraptable}
%We first show that our forward rule modification in Eq.~\eqref{eq:ResNet} does not impact the accuracy of the ResNet in the untied weights setting.
%Furthermore, we show that in a tied weights settings, this scaling is crucial to ensure good performance, when using the same learning procedure as in the untied weight setting.
We first train a ResNet-101 \citep{he2016deep} on CIFAR-10 and ImageNet using the same hyper-parameters. Experimental details are in appendix \ref{app:exp_details}  and results are summarized in table \ref{tab:perf_101}, showing that the explicit addition of the step size $\frac1N$ does not affect accuracy.
\reb{In strike contrast, the classical ResNet rule without the scaling $\frac 1N$ makes the network behave badly at large depth, while it still works well with our scaling $\frac1N$, as shown in Figure  \ref{fig:adjoint} (a).}
%\begin{figure}[H]
%\centering
%\includegraphics[width=0.4\columnwidth]{cifar_autonomous-crop.pdf} 
%\caption{\textbf{Test accuracy on CIFAR-10} as a function of the number of %blocks in each layer of the ResNet. Within each layer, weights are tied %(median over 3 runs).}\label{fig:cifar_autonomous}
%\vspace{-1em}
%\end{figure}
 On ImageNet, the scaling $\frac1N$ also leads to similar test accuracy in the weight tied setting: $72.5 \%$ with $4$ blocks per layer, $73.2 \%$ with $8$ blocks per layer and $72 \%$ with $16$ blocks per layer (mean over $2$ runs).
%\subsection{ResNets and Neural ODEs} We here illustrate that, in general, the discrete trajectory of the ResNet \ref{eq:ResNet} can be far from the continuous one of the Neural ODE \ref{eq:Neural_ode}.  We already gave an example in Figure \ref{fig:intro}: on the left, trajectories are close to an ODE, on the right, trajectories are not smooth and intersect. Second, we focus on a synthetic classification task where we want to separate two nested rings \citep{dupont2019augmented}.  As already noted \citep{teshima2020universal}, Neural ODEs ultimately fail to break apart nested rings because they cannot capture non-homeomorphic dynamics. However, we show that no matter how deep our model is, it succeeds in separating the nested rings.\pierre{need figure?}
%\subsection{Illustration our main Th. \ref{th:limit_map}}
\subsection{Adjoint method}
\paragraph{New training strategy.}
Our results in Prop. \ref{prop:reconstruction_error} and \ref{prop:gradient_error} assume uniform bounds in $N$ on our residual functions and their derivatives. We also formally proved in the linear setting that these assumptions hold during the whole learning process if the initial loss is small. 
A natural idea to start from a small loss is to consider a pretrained model. 
\begin{wraptable}[5]{r}{0.45\linewidth}
\vspace{-6mm}
	\caption{Test accuracy (ResNet)}
	\label{tab:perf_auto}
	\centering
	\scalebox{0.9}{
		\begin{tabular}{lcccc}
			\toprule
			\multicolumn{1}{l}{} & Before F.T. & After F.T \\
			\midrule
			\multicolumn{1}{l}{CIFAR-10} & $95.25 \pm 0.2$ \% & $95.65 \pm 0.1$ \%  \\
			\midrule
			\multicolumn{1}{l}{ImageNet} & 73.1 \% & 75.1 \%  \\
		\end{tabular}
	}
\end{wraptable}
In addition, we also want our pretrained model to verify assumption \ref{asp:bound} so we consider the following setup. On CIFAR (resp. ImageNet) we train a ResNet with 4 (resp. 8) blocks in each layer, where weights are tied within each layer.
A first observation is that one can transfer these weights to deeper ResNets without significantly affecting the test accuracy of the model: it remains above $94.5 \%$ on CIFAR-10 and $72 \%$ on ImageNet. We then \emph{untie the weights} of our models and refine them.
%In addition, since $\theta_n^N = \theta$, any bounds on the residuals and their derivative do not depend on $N$, so that Prop. \ref{prop:reconstruction_error} and \ref{prop:gradient_error} hold.
More precisely, for CIFAR, we then transfer the weights of our model to a ResNet with $4$, $4$, $64$ and $4$ blocks within each layer and fine-tune it only by refining the third layer, using our adjoint method.
We display in table \ref{tab:perf_auto} the median of the new test accuracy, over $5$ runs for the initial pretraining of the model. For ImageNet, we transfer the weights to a ResNet with $100$ blocks per layer and fine-tune the whole model with our adjoint method for the residual layers.
Results are summarized in table \ref{tab:perf_auto}. To the best of our knowledge, this is the first time a Neural-ODE like ResNet achieves a test-accuracy of $75.1 \%$ on ImageNet.
\paragraph{Failure in usual settings.}
In Prop. \ref{prop:reconstruction_error} we showed under assumption \ref{asp:bound}, that is if the residuals are bounded and Lipschitz continuous with constant independent of the depth $N$, then the error for computing the activations backward would scale in $\frac1N$ as well as the error for the gradients (Prop. \ref{prop:gradient_error}).
First, this results shows that the architecture needs to be deep enough, because it scales in $\frac1N$: for instance,  we fail to train a ResNet-101 \citep{he2016deep} on the ImageNet dataset using the adjoint method on its third layer (depth $23$), as shown in Figure \ref{fig:adjoint} (b).
\paragraph{Success at large depth.}
To further investigate the applicability of the adjoint method for training deeper ResNets, we train a simple ResNet model on the CIFAR data set. First, the input is processed by a $5 \times 5$ convolution with $16$ out channels, and the image is down-sampled to a size $10 \times 10$.
\begin{wrapfigure}{r}{0.4\textwidth}
\includegraphics[width=0.4\textwidth]{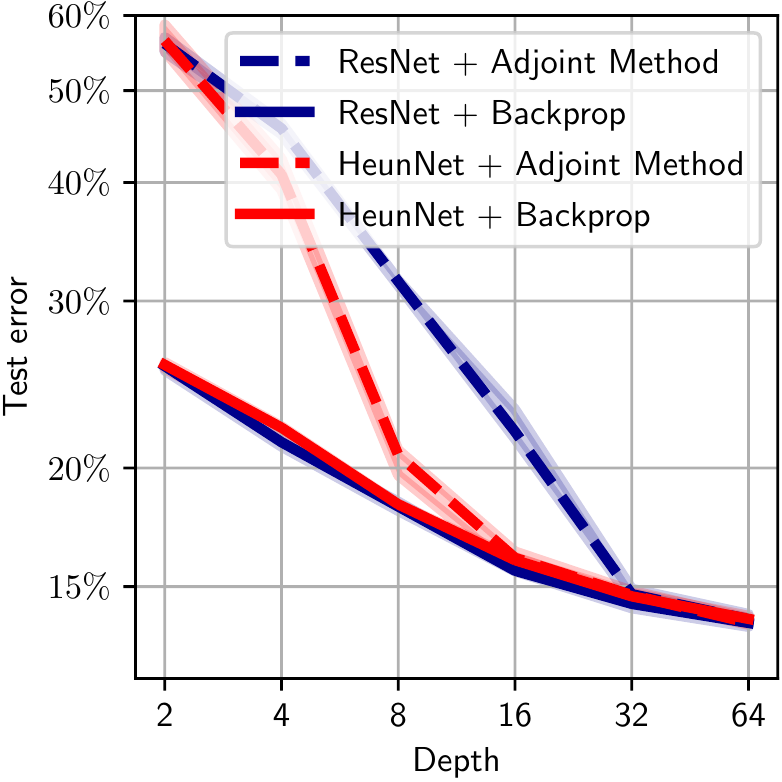}
\caption{\small{\textbf{Comparison} of the best test errors as a function of depth when using Euler or Heun's discretization method with or without the adjoint method.}}\label{fig:heun}
\vspace{-10pt}
\end{wrapfigure}
We then apply a batch norm, a ReLU and iterate relation \eqref{eq:ResNet} where $f$ is a pre-activation basic block \citep{he2016identity}. We consider the zero residual initialisation: the last batch norm of each basic block is initialized to zero. We consider different values for the depth $N$ and notice that in this setup, the deeper our model is, the better it performs in term of test accuracy. We then compare the performance of our model using a ResNet (forward rule \eqref{eq:ResNet}) or a HeunNet (forward rule \eqref{eq:heun_fwd}). 
We train our networks using either the classical backpropagation or our corresponding proxys using the adjoint method (formulas \eqref{eq:backward_pass} and \eqref{eq:heun_bwd}). We display the final test accuracy (median over $5$ runs) for different values of the depth $N$ in Figure \ref{fig:heun}. The true backpropagation gives the same curves for the ResNet and the HeunNet. 
Approximated gradients, however, lead to a large test error at small depth, but give the same performance at large depth, hence confirming our results in Prop. \ref{prop:gradient_error} and \ref{prop:gradient_error_heun}.
\reb{In addition, at fixed depth, the accuracy when training a HeunNet with the adjoint method is better (or similar at depths $2$, $32$ and $64$) than for the ResNet with the adjoint method.} This is to be linked with the two different bounds in  Prop. \ref{prop:gradient_error} and \ref{prop:gradient_error_heun}: for the HeunNet, smoothness with depth, which is expected at large depth, according to the theoretical results for the linear case (Prop. \ref{prop:smoothness}), implies a faster convergence to the true gradients for the HeunNet than for the ResNet. We finally validate this convergence in Figure \ref{fig:adjoint} (c): the deeper the architecture, the better the approximation on the gradients. In addition, the HeunNet approximates the true gradient better than the ResNet.
    \section*{Conclusion, limitations and future works}
We propose a methodology to analyze how well a ResNet discretizes a Neural ODE. The positive results predicted by our theory in the linear case are also observed in practice with real architectures: one can successfully use the adjoint method to train ResNets (or even more effectively HeunNets) using very deep architectures on CIFAR, or fine-tune them on ImageNet, without memory cost in the residual layers. However, we also show that for large scale problems such as ImageNet classification from scratch, the adjoint method fails at usual depths. 

Our work provides a theoretical guarantee for the convergence to a Neural ODE in the linear setting under a small loss initialization. A natural extension would be to study the non-linear case. In addition, the adjoint method is time consuming, and an improvement would be to propose a cheaper method than a reverse mode traversal of the architecture for approximating the activations.
\section*{Acknowledgments}

This work was granted access to the HPC resources of IDRIS under the allocation 2020-[AD011012073] made by GENCI. This work was supported in part by the French government under management of Agence Nationale de la Recherche as part of the “Investissements d’avenir” program, reference ANR19-P3IA-0001 (PRAIRIE 3IA Institute). This work was
supported in part by the European Research Council (ERC project NORIA). M. S. thanks Mathieu Blondel and Zaccharie Ramzi for helpful discussions. 

%\newpage
\bibliographystyle{unsrtnat}
\bibliography{samples}

\newpage
\appendix

\section*{\centering APPENDIX}

In Section~\ref{app:proofs} we give the proofs of all the propositions, lemmas and the theorem presented in this work.

Section~\ref{app:exp_details} gives details for the experiments in the paper.

We also give a recap on ResNet architectures in Section \ref{app:resnet_recap}.

\section{Proofs}\label{app:proofs}

\subsection{Proof of Prop. \ref{prop:approx_error}}\label{proof:approx_error}
Our proof is inspired by \citep{demailly2016analyse}.
\begin{proof}
We denote $h = \frac{1}{N}$ and $s_n = n h$. We define 
$$\varepsilon_n = x(s_{n+1}) - x(s_n) - h\phi_{\Theta}(x(s_n), s_n).$$ We have that  $\phi_{\Theta}(x(s_n), s_n) = \dot{x}(s_n)$. 

Taylor's formula gives 
$$x(s_n + h) = x(s_n) + h \dot{x}(s_n) + R_1(h)$$
with $ \| R_1(h) \| \leq \frac{1}{2}h^2\| \ddot{x}\|_{\infty}$. 
  This implies that  $$\|\varepsilon_n\| \leq \frac{1}{2}h^2 \| \ddot{x} \|_{\infty}.$$
The true error we are interested in is the global error $e_n = x(s_n) - x_n$. One has 
$$
e_{n+1} - e_n = x(s_{n+1}) - x(s_n) + x_n - x_{n+1}  = \varepsilon_n + h(\phi_{\Theta}(x(s_n), s_n) - \phi_{\Theta}(x_n, s_n)).
$$ 
Because $\phi_{\Theta}$ is $L$-Lipschitz, this gives $\|e_{n+1} - e_n\| \leq \|\varepsilon_n\| + hL\|e_n\| $ and hence 
$$
\|e_{n+1}\| \leq (1 + hL)\|e_n\| +  \|\varepsilon_n\|.
$$
Because $h = \frac 1N$, we have 
$$
\|e_{n+1}\| \leq (1 + \frac LN)\|e_n\| +  \frac{1}{2N^2}\|\ddot{x}\|_{\infty}.
$$
this implies from the discrete Gronwall lemma, since $e_0 = 0$ that 
$$
\|e_n\| \leq \frac{e^L-1}{2NL}\|\ddot{x}\|_{\infty}.
$$ 
Note that we have $\ddot{x} = \partial_s \phi_{\Theta} + \partial_x \phi_{\Theta} [\phi_{\Theta}].$ This gives the desired result. 
\end{proof}
\subsection{Proof of Prop. \ref{prop:smoothness}}\label{proof:smoothness}
\begin{proof}
Recall that we denote $\Pi^N = \prod_{n=1}^N (I_d + \frac{\theta_n^N}{N})$, $\Pi^N_{:n} = (I_d + \frac{\theta_N^N}{N})\dots (I_d + \frac{\theta_{n+1}^N}{N})$ and  $\Pi^N_{n:} = (I_d + \frac{\theta_{n-1}^N}{N})\dots (I_d + \frac{\theta_{1}^N}{N})$.
We denote $\nabla^N_n = \nabla_{\theta_n^N}L$. One has $$N \nabla^N_n =\Pi_{:n}^{N\top}(\Pi^N - B)\Sigma\Pi_{n:}^{N\top}.$$ One has as in \citep{zou2020global} that
$$\sigma^2_{max}(\Pi^N_{n:})\sigma^2_{max}(\Pi^N_{:n})\|\Pi - B\|_{\Sigma}^2 M \geq N^2\|\nabla^N_n\|^2 \geq \sigma^2_{min}(\Pi^N_{n:})\sigma^2_{min}(\Pi^N_{:n})\|\Pi - B\|_{\Sigma}^2m.$$
where $\sigma_{max}(A)$ (resp.  $\sigma_{min}(A)$) denotes the largest (resp. smallest) singular value of $A$.
We first show that $\forall t \in \RR_+$, $\|\theta_n^N(t)\| < \frac{1}{2}$. Denote $$t^* = \inf\{t\in\RR_+, \exists n \in[N-1], \|\theta_n^N(t) \| \geq 1/2 \}.$$
One has that $\forall t \in [0, t^*]$, $\sigma^2_{min}(\Pi^N_{:n}) \ge(1-\frac{1}{2N})^{2(N-n)}$ and $\sigma^2_{min}(\Pi^N_{n:}) \ge(1-\frac{1}{2N})^{2(n-1)}$ which implies that 
$$N^2\|\nabla^N_n(t)\|^2 \geq 2(1 - \frac{1}{2N})^{2N-2}{\ell}^N(t)m\geq \frac2e {\ell}^N(t)m.$$
Similarly one has $N^2 \|\nabla^N_n(t)\|^2 \leq 2e{\ell}^N(t)M$. To summarize, we have the PL conditions for $t\in[0, t^*]$:
$$
\frac2e m{\ell}^N(t) \leq N^2 \|\nabla^N_n(t)\|^2 \leq 2eM{\ell}^N(t).
$$
As a consequence, one has $$\frac{\dd {\ell}^N}{\dd t}(t) = - N \sum_{n=1}^{N} \| \nabla^N_n(t)\|^2 \leq -\frac{2}{e}m{\ell}^N(t)$$
and thus ${\ell}^N(t) \leq e^{-\frac2emt}{\ell}^N(0)$. 

We have $\theta^N_n(t^*) = \theta^N_n(0) + N \int_0^{t^*} \nabla^N_n$ and $\| \nabla^N_n \| \leq \frac{\sqrt{2eM}}{N} \sqrt{{\ell}^N}$ so that 
$$
\|\theta^N_n(t^*)\| \leq \| \theta^N_n(0)\| + \sqrt{2eM}\int_{0}^{t^*} e^{-\frac1emt}\sqrt{{\ell}^N(0)}dt < \frac14 + \frac14 < 1/2.
$$

This is absurd by definition of $t^*$ and thus shows that $\forall t\in\RR_+, \|\theta_n^N(t)\| < \frac{1}{2}.$  We also see that $\nabla^N_n$ is integrable so that $\theta_n^N(t)$ admits a limit as $t \to \infty$.

We now show our main result. Note that we have the relationship $(I+\frac{\theta_{n+1}^N}{N})^\top\nabla_{n+1} = \nabla_{n}(I+\frac{\theta_{n}^N}{N})^\top$ so that 
$$\nabla^N_{n+1} - \nabla^N_n = (I + \frac{\theta^{N\top}_{n+1}}{N})^{-1}(\frac{\nabla^N_n\theta_n^{N\top} - \theta_{n+1}^{N\top} \nabla_n^N}{N}).
$$
Because $\|(I +A)^{-1} \| \leq 2$ if $\|A\|\leq \frac12$ this gives $\|\nabla^N_{n+1} - \nabla^N_n \| \leq \frac2N\|\nabla^N_n\|$. Integrating we get 
$$\|\theta^N_{n+1}(t) - \theta^N_n(t) \| \leq \|\theta^N_{n+1}(0) - \theta^N_n(0) \| + 2 \int_0^t \|\nabla^N_n\| .
$$
This gives 
$$
\|\theta^N_{n+1}(t) - \theta^N_n(t) \| \leq O(\frac1N) + \frac{1}{N}2\int_0^t\sqrt{2eM{\ell}^N(0)}e^{-\frac1emt}dt = O(\frac{1}{N}),
$$
which is the desired result. 

\end{proof}
\subsection{Proof of lemma \ref{lemm:limit_functions}}\label{proof:limit_functions}
\begin{proof}
We adapt a variant of the Ascoli–Arzelà theorem \citep{brezis2011functional}. 
We showed in Prop. \ref{prop:smoothness} that there exists $C > 0$ that only depends on the initialization such that, $\forall t\geq0, \forall i \in [N-1]$, $$\|\theta^N_{n+1}(t) - \theta^N_n(t) \| \leq \frac CN.$$ 
This implies that $$\|\theta^N_{j}(t) - \theta^N_i(t) \| \leq C\frac {|j-i|}N.$$
We also have that $$\|\theta^N_{n}(t_1) - \theta^N_n(t_2) \| = \|N\int_{t_1}^{t_2}\nabla^N_n\| \leq C'|t_1 - t_2|$$ with $C' \ge 0 $. 

Its follows that  $\| \psi_N(s_1, t_1) - \psi_N(s_2, t_2) \| \leq \| \psi_N(s_1, t_1) - \psi_N(s_1, t_2) \| + \| \psi_N(s_1, t_2) - \psi_N(s_2, t_2) \|$  and thus 
$$(i) \quad \| \psi_N(s_1, t_1) - \psi_N(s_2, t_2) \| \leq C'|t_1-t_2| + C|s_1-s_2| + \frac{C}{N}.$$

We also have  
$$ (ii) \quad \forall N\in \NN, \quad \|\psi_N \|_{\infty} \leq \frac12.$$ These two properties are essential to prove our lemma. We proceed as follows.

1) First, we denote $((s_j, t_j))_{j\in\NN} = (\QQ \cap [0,1]) \times \QQ_+$. Since we have the uniform bound $(ii)$, we extract using a diagonal extraction procedure a subsequence $\psi_{\sigma(N)}$ such that $\forall j \in \NN$, $$\psi_{\sigma(N)}(s_j, t_j) \to \psi(s_j, t_j)$$ (we denote the limit $\psi(s_j, t_j)$).

2) We show the convergence $\forall s \in [0, 1]$ and $t \in \RR_+$.

Let $\varepsilon >0$, $s\in [0, 1]$ and $t \in \RR_+$. Since $((s_j, t_j))_{j\in\NN}$ is dense in $[0, 1] \times \RR_{+}$, there exists $k \in \NN$ such that $|s_k - s| < \varepsilon$ and $|t_k - t| < \varepsilon$. Let $N, M \in \NN$. 

We have  

$\|\psi_{\sigma(N)}(s, t) - \psi_{\sigma(M)}(s, t)\| \leq \|\psi_{\sigma(N)}(s,t ) - \psi_{\sigma(N)}(s_k, t_k)\| + \|\psi_{\sigma(N)}(s_k, t_k) - \psi_{\sigma(M)}(s_k, t_k)\| +\|\psi_{\sigma(M)}(s_k, t_k) - \psi_{\sigma(M)}(s, t)\|$

so that 
$$\|\psi_{\sigma(N)}(s, t) - \psi_{\sigma(M)}(s, t)\| \leq 2 C \varepsilon + 2 C'\varepsilon + \frac{C}{\sigma(N)} + \frac{C}{\sigma(M)} + \|\psi_{\sigma(N)}(s_k, t_k) - \psi_{\sigma(M)}(s_k, t_k)\|.$$
Since $(\psi_{\sigma(N)}(s_k, t_k))_{N\in\NN}$ is a Cauchy sequence, this gives for $N, M$ big enough that
$$\|\psi_{\sigma(N)}(s) - \psi_{\sigma(M)}(s, t)\| \leq (2(C+C') + 1) \varepsilon$$
and thus $(\psi_{\sigma(N)}(s, t))$ is a Cauchy sequence in $\RR^{d\times d}$. As such, it converges and one has $$\psi_{\sigma(N)}(s, t) \to \psi(s, t).$$ 

3) Recall that one has 
$$\| \psi_{\sigma(N)}(s_1, t_1) - \psi_{\sigma(N)}(s_2, t_2) \| \leq C|s_1-s_2| + \frac{C}{\sigma(N)} + C'|t_1-t_2|$$
so that letting $N \to \infty$ gives $$
\|\psi(s_1, t_1) - \psi(s_2, t_2)\| \leq C|s_1-s_2| + C'|t_1-t_2|$$
and $\psi$ is Lipschitz continuous. 

4) Let us finally show that the convergence is uniform in $(s,t)$. Let $s \in [0, 1]$, $\varepsilon >0 $ and $\delta >0$ such that if $|s-u|<\delta$, $\forall t \in \RR+$, 
$$\| \psi_N(s, t) - \psi_N(u, t) \| \leq \varepsilon + \frac{C}{N}$$ and $\| \psi(s, t) - \psi(u, t) \| \leq \varepsilon $. There exists a finite set of $\{s_j\}_{j=1}^{k}$ such that $$[0, 1] \subset \cup_{j=1}^k ]s_j - \frac{\delta}{2}, s_j + \frac{\delta}{2}[.$$ For our $s$, there exists $j \in \{ 1, \dots, k\}$ such that $\|s -s_j \| \leq \delta$.

There also exists $t_0 \geq 0$ such that if $t \geq t_0$, 
$$\|\psi_{\sigma(N)}(s, t) -  \psi_{\sigma(N)}(s, t_0)\| \leq \sqrt{2eM}\int_{t_0}^t e^{-\frac1em z}\sqrt{{\ell}^N(0)}dz \leq \varepsilon.$$
We have:

$\|\psi_{\sigma(N)}(s, t_0) - \psi(s, t_0)\| \leq  \|\psi_{\sigma(N)}(s, t_0) - \psi_{\sigma(N)}(s_j, t_0)\| + \|\psi_{\sigma(N)}(s_j, t_0) - \psi(s_j, t_0)\| + \|\psi(s_j, t_0) - \psi(s, t_0)\|$ 

and thus:

$\|\psi_{\sigma(N)}(s, t_0) - \psi(s, t_0)\| \leq 2\varepsilon + \frac{C}{\sigma(N)} + \max_{j \in \{ 1, \dots, k\}}\|\psi_{\sigma(N)}(s_j, t_0) - \psi(s_j, t_0)\| \leq 4 \varepsilon$ for $N$ big enough.

Finally, 
$\|\psi_{\sigma(N)}(s, t) - \psi(s, t)\| \leq \|\psi_{\sigma(N)}(s, t) - \psi_{\sigma(N)}(s, t_0)\| + \|\psi_{\sigma(N)}(s, t_0) - \psi(s, t_0)\| +\|\psi(s, t_0) - \psi(s, t)\| \leq 6 \varepsilon$ 

for $N$ big enough, independently of $t$ and $s$. This concludes the proof.
\end{proof}
\subsection{Proof of lemma \ref{lemm:closeness}}\label{proof:closeness}
\begin{proof}

We group terms $2$ by $2$ in the product $\Pi^{2N}$. One has $(I + \frac{\theta^{2N}_{2n}}{2N})(I + \frac{\theta^{2N}_{2n-1}}{2N}) = (I + \frac{\tilde{\theta}^{N}_{n}}{N})$ with 
$$\tilde{\theta}_n^N = (\frac{\theta^{2N}_{2n} + \theta^{2N}_{2n-1}}{2} + \frac{\theta^{2N}_{2n}\theta^{2N}_{2n-1}}{4N}),$$
So that $\Pi^{2N} = \tilde{\Pi}^N$ where $\tilde{\Pi}^N$ is defined as $\Pi^N$ with $\tilde{\theta}_n^N$.
One has by Prop. \ref{prop:smoothness} that
$$\tilde{\theta}_n^N = {\theta}_{2n}^{2N} + O(\frac1N).$$
We will show that $\tilde{\theta}_n^N ={\theta}_n^N + O(\frac1N).$

Let $D^N_n = \|\theta^N_n - \tilde{\theta}^N_{n}\|$ and $D^N = \frac1N \sum_{n=1}^N D_n.$ We have $$2D^N_n\dot{D}^N_n = -N\langle\nabla^N_n - \tilde{\nabla}^N_n, \theta^N_n - \tilde{\theta}^N_{n}\rangle.$$
In addition, we have 
$$
N(\nabla^N_n - \tilde{\nabla}^N_{n}) = \Pi_{:n}^{N\top}(\Pi^N - B)\Sigma\Pi_{n:}^{N\top} - \tilde{\Pi}_{:n}^{N\top}(\tilde{\Pi}^N - B)\Sigma\tilde{\Pi}_{n:}^{N\top}
$$
so that
$$
N(\nabla^N_n - \tilde{\nabla}^N_{n})= (\Pi^N_{:n} - \tilde{\Pi}_{:n}^N)^{\top}(\Pi^N - B)\Sigma\Pi_{n:}^{N\top} +  \tilde{\Pi}_{:n}^{N\top}(\Pi^N - B)\Sigma(\Pi^N_{n:} - \tilde{\Pi}_{n:}^N)^{\top} + \tilde{\Pi}_{:n}^{N\top}(\Pi^N - \tilde{\Pi}^N)\Sigma\tilde{\Pi}_{n:}^{N\top}.
$$
Note also that since the Jacobian of $(\theta_1, ..,\theta_N) \to \Pi^N$ is 
$$
J_{(\theta_1, .., \theta_N)}(H_1, .., H_N) = \frac1N \sum_{n=1}^N\Pi^N_{:n}H_n\Pi^N_{n:}
$$
and the $\theta^N_n$'s are such that $\|\theta^N_n\| \leq \frac12$, there exists a constant $K >0$ such that $\|\Pi^N_{:n} - \tilde{\Pi}^N_{:n}\| \leq K D^N$.
Again because $\|\theta^N_n\| \leq \frac12$ and  $\|\tilde{\theta}^N_n\| \leq \frac12$, this gives 
$$
N\| \nabla_n - \tilde{\nabla}_{n} \| \leq \alpha K D^N \sqrt{{\ell}^N} + \beta \| \Pi^N- \tilde{\Pi}^N \| 
$$
for some constants $\alpha$, $\beta$.
Finally, we have 
$$
\dot{D}^N_n \leq \frac12 (\alpha K D^N \sqrt{{\ell}^N} + \beta \| \Pi^N- \tilde{\Pi}^N \| )
$$
which gives $\forall t$
\begin{equation}\label{eq:D_n}
2{D}^N_n(t) \leq \alpha K \int_0^t D^N \sqrt{{\ell}^N} + \beta  \int_0^t \| \Pi^N- \tilde{\Pi}^N \| + O(\frac1N).
\end{equation}
We now focus on the $\beta$ term involving $\| \Pi^N- \tilde{\Pi}^N \|$.
Denote $\Delta^N =  \Pi^N - \tilde{\Pi}^N  $. One has
$$
 \dot{\Delta}^N = -\frac1N(\sum_{n=1}^N\Pi^N_{:n}\Pi_{:n}^{N\top}(\Pi^N - B)\Sigma\Pi_{n:}^{N\top}\Pi^N_{n:} + \tilde{\Pi}^N_{:n}\tilde{\Pi}_{:n}^{N^\top}(\tilde{\Pi}^N - B)\Sigma\tilde{\Pi}_{n:}^{N\top}\tilde{\Pi}^N_{n:}),
$$
and equivalently:

$
 \dot{\Delta}^N = -\frac1N(\sum_{n=1}^N[\Pi^N_{:n}\Pi^{N\top}_{:n} - \tilde{\Pi}^N_{:n}\tilde{\Pi}_{:n}^{N\top}](\Pi^N - B)\Sigma\Pi_{n:}^{N\top}\Pi^N_{n:}
 + \tilde{\Pi}^N_{:n}\tilde{\Pi}_{:n}^{N\top}({\Pi^N} - B)\Sigma[\Pi_{n:}^{N\top}\Pi^N_{n:} - \tilde{\Pi}_{n:}^{N\top}\tilde{\Pi}^N_{n:}] + \tilde{\Pi}^N_{:n}\tilde{\Pi}_{:n}^{N\top}({\Pi}^N - \tilde{\Pi}^N)\Sigma \tilde{\Pi}_{n:}^{N\top}\tilde{\Pi}^N_{n:}).
$

Note that similarly to $\|\Pi^N - \tilde{\Pi}^N\|$ there exist $K'$ such that $\|\Pi^N_{:n}\Pi^{N\top}_{:n} - \tilde{\Pi}^N_{:n}\tilde{\Pi}_{:n}^{N\top}\| \leq K'D$
so that 
$$\|\dot{\Delta}^N + \frac1N \sum_{n=1}^N \tilde{\Pi}^N_{:n}\tilde{\Pi}_{:n}^{N\top} \Delta^N \tilde{\Pi}_{n:}^{N\top}\tilde{\Pi}^N_{n:}\| \leq a K' D \sqrt{\ell^N}.$$ 
Let us denote by $H$ the operator: $$H(\Delta) = \frac1N \sum_{n=1}^N \tilde{\Pi}^N_{:n}\tilde{\Pi}_{:n}^{N\top} \Delta \Sigma \tilde{\Pi}_{n:}^{N\top}\tilde{\Pi}^N_{n:}.$$ 
Our (PL) conditions precisely write $-\Delta^\top H(\Delta) \leq - \lambda \|\Delta\|^2$ for some $\lambda > 0$.
Let $\phi^N = \frac12 \|\Delta^N\|^2.$ One has 
$$
\frac{\dd \phi^N}{\dd t} = \langle\Delta^N,\dot{\Delta}^N + H(\Delta^N)\rangle -   \langle\Delta^N,H(\Delta^N)\rangle
$$
so that 
$$
\frac{\dd \phi^N}{\dd t} \leq (aK' D \sqrt{\ell^N} )\sqrt{2\phi^N} - 2 \lambda \phi^N.
$$
Since $\|\Delta^N\| = \sqrt{2\phi^N}$ we get
$$\frac{\dd\|\Delta^N\|}{\dd t} = \frac{\frac{\dd \phi^N}{\dd t}}{\sqrt{2\phi^N}}.$$
We finally have
$$
\frac{\dd \|\Delta^N\|}{\dd t} \leq a K' D^N \sqrt{\ell^N} - \lambda \|\Delta^N\|.
$$
Integrating, we get 
$$
\|\Delta^N(t)\| \leq - \lambda \int_0^t\|\Delta^N\| + \int_0^taK'D^N\sqrt{l^N} + O(\frac1N)
$$
and then
$$
\int_0^t\|\Delta^N\| \leq \frac{1}{\lambda} \int_0^taK'D^N\sqrt{l^N} + O(\frac1N).
$$
Plugging this into \eqref{eq:D_n} leads to
$$
0 \leq 2{D}^N_n(t) \leq \alpha K \int_0^t D^N \sqrt{{\ell}^N} + \frac{\beta}{\lambda}  \int_0^t aK'D^N\sqrt{\ell^N} + O(\frac1N).
$$
Let $n(t)$ be such that $D^N_{n(t)}(t) = \max_{i\in[1, N]}D^N_i(t)$.
We have
$$
0 \leq 2{D}^N_{n(t)}(t) \leq \mu \int_0^t D^N_{n(\tau)}(\tau) \sqrt{{\ell}^N(\tau)}d\tau + O(\frac1N)
$$
for some constant $\mu >0$.
And since $\sqrt{{\ell}^N}$ is integrable, we get by Gronwall's inequality that $D^N_n = O(\frac1N)$ $\forall n \in [1, N]$.
We showed: 
$$
\theta^{2N}_{2n} =  \theta^{N}_{n} + O(\frac1N).
$$
\end{proof}

\subsection{Proof of Th. \ref{th:limit_map}}\label{proof:limit_map}

We first prove the following lemma \ref{lem:pi_n} before proving Th.  \ref{th:limit_map}.
\begin{lem}\label{lem:pi_n}
Under the assumptions of Th.  \ref{th:limit_map}, let $\sigma$ be such that $\psi_{\sigma(N)} \to \psi_{\sigma}$ uniformly (in $\|.\|_{\infty}$ w.r.t $(s,t)$). Then one has $\Pi^{\sigma(N)}(t) \to \Pi(t)$ uniformly (in $t$) where $\Pi(t)$ maps $x_0$ to the solution at time $1$ of the Neural ODE $\frac{\dd x }{\dd s} = \psi_{\sigma}(s, t)x(s)$ with initial condition $x_0.$
\end{lem}
\begin{proof}
Consider for $x_0 \in \RR^d$ with $\|x_0 \|=1$ the discrete scheme
$$
x_{n+1} = x_n +\frac1{\sigma(N)}\theta^{\sigma(N)}_n(t)x_n,
$$
the ODE
$$
\frac{\dd x }{\dd s} = \psi_{\sigma}(s, t)x(s),
$$
and the Euler scheme with time step $\frac1{\sigma(N)}$ for its discretization
$$
y_{n+1} = y_n +\frac1{\sigma(N)}\psi_{\sigma}(\frac{n}{\sigma(N)}, t)y_n.
$$
We know by Prop. \ref{prop:approx_error}, since $x_0$ has unit norm that
$$
\|x(\frac{n}{\sigma(N)}) - y_n \| \leq \frac{e^{\frac12}-1}{\sigma(N)} \| \partial_s \psi_{\sigma}(., t) + \psi_{\sigma}^2(., t)\|^{{K}\times[0, 1]}_{ \infty}$$
where $K$ is a compact that contains all the trajectory starting from any unit norm initial condition. 
Since $\forall t \in \RR_+$, $\| \partial_s\psi_{\sigma}(s,t)\| \leq C $ and $\|\psi_{\sigma}(s,t)^2\| \leq \frac12$, there exists $\tilde{C} >0$ and independent of $t$ such that 
$$
\|x(\frac{n}{\sigma(N)}) - y_n \| \leq \frac{\tilde{C}}{\sigma(N)}
$$
Now, let $e_n = y_{n} - x_n$. We have 
$$
e_{n+1} = e_n(1 + \frac1{\sigma(N)}\psi_{\sigma}(\frac{n}{\sigma(N)}, t)) + \frac1{\sigma(N)}(\psi_{\sigma}(\frac{n}{\sigma(N)}, t) - \psi_{\sigma(N)}(\frac{n}{\sigma(N)},t))x_n.
$$
Since $\|\theta^N_n\| \leq \frac12$ and $x_0$ has unit norm, there exists $M >0$ independent of $x_0$ such that, $\forall n$ and $N$, $\|x_n\| \leq M$.
Thus 
$$
\|e_{n+1}\| \leq \|e_n\|(1 + \frac1{2\sigma(N)}) + \frac1{\sigma(N)}\sup_{(s,t)\in[0,1]\times\RR_{+}}\|\psi_{\sigma}(s,t) - \psi_{\sigma(N)}(s,t)\|M.
$$
The fact that $\sup_{(s,t)\in[0,1]\times\RR_{+}}\|\psi_{\sigma}(s,t) - \psi_{\sigma(N)}(s,t)\| \to 0$ (uniform convergence of $\psi_{\sigma(N)}$ to $\psi_{\sigma}$) along with the discrete Gronwall's lemma leads to $\|e_n\| = o(1)$ independent of $t$ and $x_0$. More precisely, 
$$
 \sup_{t\in \RR_+, x_0\in\RR^d, \|x_0\|=1}\|\Pi^{\sigma(N)}(t)x_0 - \Pi(t) x_0\| \to 0
$$
as $N \to \infty$. We obtain the uniform convergence with $t$.

\end{proof}
We can now prove our Th. \ref{th:limit_map}.
\begin{proof}
Consider $(\psi_{\sigma(N)})_N$ a sub-sequence of $(\psi_N)_N$ as in lemma \ref{lemm:limit_functions} that converges to some $\psi_{\sigma}$.

1) We first prove the uniqueness of the limit. 

We want to show that $\psi_{\sigma}$ does not depend on $\sigma$. This will imply the uniqueness of any accumulation point of the relatively compact sequence $(\psi_N)_N$ and thus its convergence.

We have $\forall s \in[0, 1]$, $$\partial_t \psi_{\sigma(N)}(s, t) = - \Pi^{\sigma(N)\top}_{:\lfloor \sigma(N)s \rfloor}(t)(\Pi^{\sigma(N)}(t) - B)\Pi^{\sigma(N)\top}_{\lfloor \sigma(N)s \rfloor:}(t).$$ As $N \to \infty$, we have thanks to lemma \ref{lem:pi_n} that the right hand term converges uniformly to
$$
- \Pi^{\top}_{:s}(t)(\Pi(t) - B)\Pi^{\top}_{s:}(t)
$$ 
where $\Pi$ maps $x_0$ to the solution at time $1$ of the Neural ODE $\frac{\dd x }{\dd s} = \psi_{\sigma}(s, t)x(s)$ with initial condition $x_0$, $\Pi_{:s}(t)$ maps $x_0$ to the solution at time $s$ of the Neural ODE $\frac{\dd x }{\dd s} = \psi_{\sigma}(s, t)x(s)$ with initial condition $x_0$ and $\Pi_{s:}(t)$ maps $x_0$ to the solution at time $1-s$ of the Neural ODE $\frac{\dd x }{\dd s} = \psi_{\sigma}(s, t)x(s)$ with initial condition $x_0$. 

This uniform convergence makes it possible to consider the limit ODE as $N \to \infty$:
\begin{equation}
    \partial_t \psi_{\sigma}(., t) = F(\psi_{\sigma}(., t)), \quad \psi_{\sigma}(., 0) = 0_{d \times d}
\end{equation}
where $\forall s \in [0,1]$, 
$$ F(\psi_{\sigma}(s, t) )= - \Pi^{\top}_{:s}(t)(\Pi(t) - B)\Pi^{\top}_{s:}(t).$$
We now show that $F$ is Lipschitz continuous which will guarantee uniqueness through the Picard–Lindelöf theorem. Recall that we have $\forall (s,t)\in [0,1] \times \RR_+$: $$\| \psi_{\sigma}(s, t) \| \leq \frac{1}{2}.$$
Let $\psi_1$, $\psi_2$ with $\| \psi_{1}(s, t) \| \leq \frac{1}{2}$ and $\| \psi_{2}(s, t) \| \leq \frac{1}{2}$ and $\Pi_1(t)$, $\Pi_2(t)$ the corresponding flows. 

Let $x_0$ in $\RR^d$ with unit norm, $x_1$ (resp. $x_2$) be the solutions of $\frac{\dd x }{\dd s} = \psi_1(s, t)x(s)$ (resp. $\frac{\dd x }{\dd s} = \psi_2(s, t)x(s)$) with initial condition $x_0$. Let $y = x_1 - x_2 $. 

One has $\Pi_1(t)x_0 = x_1(1)$ and  $\Pi_2(t)x_0 = x_2(1)$. One has $\dot{y} = \psi_1 x_1 - \psi_2 x_2 = \psi_2 y + (\psi_1 - \psi_2)x_1$. Hence, since $y(0) = 0$, $\|y(s)\| \leq \int_0^s\|\psi_2\| \|y\| +\|\psi_1 - \psi_2\|_{\infty}| \|x_1\|_{\infty}$, we have $$\|y(s)\| \leq \frac12 \|y(s)\| + \|\psi_1 - \psi_2\|_{\infty} . \|\Pi_1(t)\|$$
and since $\forall t \in \RR_+$, $\|\Pi_1(t)\| \leq 2e$ we get 
$$\|\Pi_1(t)x_0 - \Pi_2(t)x_0\| = \|y(1)\| \leq \alpha \|\psi_1 - \psi_2\|_{\infty}$$ for some $\alpha > 0$. The same arguments go for $\Pi_{:s}$ and $\Pi_{s:}$. 

Since we only consider maps $\psi_{\sigma}$ such that $\|\psi_{\sigma}(s, t)\| \leq \frac12$, this implies that the product is also Lipschitz and thus $F$ is Lipschitz. This guarantees the uniqueness of a solution $\psi$ to the Cauchy problem and we have that $\psi_N \to \psi$ uniformly.

2) We now turn to the convergence speed. 

We have $\|\psi_{2N} -\psi_{N}\| \leq \frac{D}{N}$ for some $D > 0$ thanks to lemma \ref{lemm:closeness}. For $k \in \NN$, we have that
$$
\|\psi_{2^kN} -\psi_{N}\| \leq \sum_{i=0}^{k-1} \|\psi_{2^{i+1}N} -\psi_{2^iN}\| \leq \frac{D}{N}\sum_{i=0}^{k-1}\frac{1}{2^i} \leq \frac{2D}{N}.$$
Letting $k \to \infty$ finally gives $\|\psi -\psi_{N}\| \leq \frac{2D}{N}.$

%Finally, this gives $\| F(\psi_1) - F(\psi_2)\| \leq 8e^2\|\psi_1 - \psi_2\|_{\infty}$ and 
\end{proof}
\subsection{Proof of Prop. \ref{prop:reconstruction_error}}\label{proof:reconstruction_error}
\begin{proof}
We denote $r_n =  \tilde{x}_{n} - x_{n}$. 

One has $r_N = 0$ and 
$$
r_n = \tilde{x}_{n+1} - \frac1N f(\tilde{x}_{n+1}, \theta^N_n) - x_{n+1} + \frac{1}{N}f({x}_{n}, \theta^N_n),
$$
that is
$$
r_n=r_{n+1} +\frac1N(f(x_{n+1} - \frac{1}{N}f({x}_{n}, \theta^N_n), \theta^N_n) - f(\tilde{x}_{n+1}, \theta^N_n)).
$$ 
Since 
$$
f(x_{n+1} - \frac{1}{N}f({x}_{n}, \theta^N_n), \theta_n^N) = f(x_{n+1}, \theta^N_n) - \frac{1}{N}\partial_x{f(x_{n+1},\theta_n^N)}[f({x}_{n}, \theta^N_n)] + O(\frac{1}{N^2})
$$
this gives 
$$r_n = r_{n+1} + \frac{1}{N}(f(x_{n+1}, \theta^N_n)-f(\tilde{x}_{n+1}, \theta^N_n)) - \frac{1}{N^2}\partial_x{f(x_{n+1},\theta_n^N)}[f({x}_{n}, \theta^N_n)] + O(\frac{1}{N^3}).$$
Denoting 
$$K_N = \sup_{n\in[N-1]}\| \partial_x{f(.,\theta_n^N)}  \|^{K}_{ \infty}\|f(.,\theta_n^N)\|^{K}_{ \infty},$$ 
we have the following inequality:
$$
\|r_n\| \leq (1 + \frac {L_f}N)\|r_{n+1}\| + \frac{1}{N^2}K_N + O(\frac{1}{N^3})
$$
and since $r_N = 0$, the discrete Gronwall lemma leads to $\|r_n\| \leq \frac{e^{L_f}-1}{{L_f}N}K_N + O(\frac{1}{N^2}).$ In addition, one has $K_N \leq L_f C_f$ so that 
$$
\|r_n\| \leq \frac{e^{L_f}-1}{N}C_f + O(\frac{1}{N^2}).
$$
\end{proof}
\subsection{Proof of Prop. \ref{prop:gradient_error}}\label{proof:gradient_error}
\begin{proof}
1) We first control the error made in the gradient with respect to activations.

Denote $$g_n = \nabla_{\tilde{x}_n}L - \nabla_{x_n}L.$$
One has using formulas \eqref{eq:backprop} and \eqref{eq:backprop_proxy} that $$g_n = g_{n+1} +\frac{1}{N}(\partial_xf(\tilde{x}_n, \theta^N_n)-\partial_xf(x_n, \theta^N_n))^\top\nabla_{\tilde{x}_{n+1}}L + \frac1N\partial_xf(x_n, \theta^N_n)^\top g_{n+1}.$$
Since  $$\|\partial_xf(x_n, \theta^N_n)^\top g_{n+1}\| \leq L_f\|g_{n+1}\|$$ and because $$\|(\partial_xf(\tilde{x}_n, \theta^N_n)-\partial_xf(x_n, \theta^N_n))^\top\nabla_{\tilde{x}_{n+1}}L\| \leq L_{df}\|\tilde{x}_n - x_n\|g,$$ where $g$ is a bound on $\nabla_{\tilde{x}_{n+1}}L$, we conclude by using Prop. \ref{prop:reconstruction_error} and the discrete Gronwall's lemma.

2) We can now control the gradients with respect to the parameters $\theta^N_n$'s.

Denote 
$$t_n = \tilde{\nabla}_{\theta^N_n}L - {\nabla}_{\theta_n^N}L.$$ 
We have 
$$Nt_n = -[\partial_{\theta}f({x}_{n},\theta_n^N) - \partial_{\theta}f(\tilde{x}_{n},\theta_n^N)]^\top \nabla_{{x}_{n}}L - [\partial_{\theta}f(\tilde{x}_{n},\theta_n^N)]^\top g_n.$$
Hence $N\|t_n\| \leq L_{\theta}\|x_n - \tilde{x}_n\|g + C_{\theta}\|g_n\|$ where $g$ is a bound on $\nabla_{x_n}L$.

Using our bound on $\|g_n\|$ and Prop. \ref{prop:reconstruction_error} we get 
$$N\|t_n\| \leq \frac{L_{\theta}(e^{L_f}-1)gC_f}{N} + \frac{(e^{L_f}-1)(e^{L_f}-1) C_f L_{df}gC_{\theta}}{L_fN} +  O(\frac{1}{N^2})$$
and thus 
$$t_n = O(\frac{1}{N^2}).$$
\end{proof}

\subsection{Proof of Prop. \ref{prop:reconstruction_error_heun}}\label{proof:reconstruction_error_heun}

In the following, we let for short $f_n(x) = f(x, \theta_n^N)$, and we define
\begin{equation}
\label{app:eq:heun_functions}
\phi_n(x) = \frac12\left(f_n(x) + f_{n+1}(x + \frac1N f_n(x))\right)\text{ and }  \psi_n(x) =  \frac12\left(f_{n+1}(x) + f_{n}(x - \frac1N f_{n+1}(x))\right)
\end{equation}

so that Heun's forward and backward equations are
$$
x_{n+1} = x_n + \frac1N \phi_n(x_n)  \text{ and } \tilde{x}_n = \tilde{x}_{n+1} -\frac1N \psi_n(\tilde{x}_{n+1}).$$

We have the following lemma that quantifies the reconstruction error over one iteration:
\begin{lem}
  \label{app:lem:heun_rec}
  For $x\in\mathbb{R}$, we have as $N$ goes to infinity
  $$
  \psi_n(x + \frac1N \phi_n(x)) - \phi_n(x) = \frac1{4N}\left(J_{n+1}(x) - J_n(x)\right)[f_{n+1}(x) - f_n(x)] + O(\frac1{N^2}),
  $$
  
where $J_n =\partial_x f_n(x)$ is the Jacobian of $f_n$.
\end{lem}
\begin{proof}
As $N$ goes to infinity, we have the following expansions of~\eqref{app:eq:heun_functions}:
$$
\phi_n(x) = \frac12(f_n(x) + f_{n+1}(x)) + \frac1{2N} J_{n+1}(x)[f_n(x)] + O(\frac1{N^2}),
$$
$$
\psi_n(x) = \frac12(f_n(x) + f_{n+1}(x)) - \frac1{2N} J_{n}(x)[f_{n+1}(x)]+ O(\frac1{N^2}).
$$
As a consequence, we have
\begin{align*}
  \psi_n(x + \frac1N\phi_n(x)) = &\frac12(f_n(x) + f_{n+1}(x)) - \frac1{2N} J_{n}(x)[f_{n+1}(x)]\\
  &+ \frac1{4N}(J_n(x)[f_n(x) + f_{n+1}(x)] +J_{n+1}(x)[f_n(x) + f_{n+1}(x)] ) +O(\frac1{N^2}). 
\end{align*}
Putting everything together, we find that the zero-th order in $\psi_n(x + \frac1N \phi_n(x)) - \phi_n(x)$ cancels, and that the first order simplifies to $\frac1{4N}\left(J_{n+1}(x) - J_n(x)\right)[f_{n+1}(x) - f_n(x)]$.
\end{proof}
We now turn the the proof of the main proposition:
\begin{proof}
  We let $r_n = \tilde{x}_n - x_n$ the reconstruction error. We have $r_N=0$, and we find
  \begin{align}
    r_n &= \tilde{x}_n - x_n\\
    &=  \tilde{x}_{n+1} -\frac1N \psi_n(\tilde{x}_{n+1}) - x_{n+1} + \frac1N \phi_n(x_n)\\
    &= r_{n+1} - \frac1N\left(\psi_n(\tilde{x}_{n+1}) - \psi_n(x_{n+1})\right) -\frac1N \left(\psi_n(x_{n+1}) - \phi_n(x_n)\right).
  \end{align}
 
Using the triangle inequality, and the $L_f'-$Lispchitz continuity of $\psi_n$, we get
$$
\|r_n\|\leq (1 + \frac{L_f'}N)\|r_{n+1}\| + \frac1N \|\psi_n(x_{n+1}) - \phi_n(x_n)\|.
$$
The last term is controlled with the previous Lemma~\ref{app:lem:heun_rec}:

\begin{align}
  \|\psi_n(x_{n+1}) - \phi_n(x_n)\| &\leq \frac1{4N}\|\left(J_{n+1}(x_n) - J_n(x_n)\right)[f_{n+1}(x_n) - f_n(x_n)]\| + O(\frac1{N^2})  \\
  &\leq \frac{C_f'\Delta_\theta^N}{N} + O(\frac1{N^2}).
\end{align}
We therefore get the recursion
$$
\|r_n\|\leq (1 + \frac{L_f'}N)\|r_{n+1}\| + \frac{C_f'\Delta_\theta^N}{N^2} + O(\frac1{N^3}).
$$

Unrolling the recursion gives,
$$
\|r_n\| \leq \frac{(e^{L_f'} - 1)C_f'}{L_f'N}\Delta_\theta^N + O(\frac1{N^2}).
$$
\end{proof}
\subsection{Proof of Prop. \ref{prop:gradient_error_heun}}\label{proof:gradient_error_heun}
\begin{proof}
1) We first control the error made in the gradient with respect to activations. We have the following recursions:
$$
\nabla_{x_{n}}L = (I + \frac1N\partial_x \phi_n(x_{n+1}))^\top \nabla_{x_{n+1}}L\text{ and }\nabla_{\tilde{x}_{n}}L = (I + \frac1N \partial_x \phi_n(\tilde{x}_{n+1}))^\top \nabla_{\tilde{x}_{n+1}}L
$$

Letting $r'_n = \nabla_{x_{n}}L  - \nabla_{\tilde{x}_{n}}L$, we have
$$
r'_n = r'_{n+1} + \frac1N\partial_x \phi_n(x_{n+1})^\top r'_{n+1} + \frac1N\left(\partial_x \phi_n(x_{n+1}) - \partial_x \phi_n(\tilde{x}_{n+1})\right)^\top\nabla_{\tilde{x}_{n+1}}L
$$

Therefore, using the triangle inequality, and letting $g$ a bound on the norm of the gradients $\nabla_{\tilde{x}_{n+1}}L$ and $\Delta$ a Lipschitz constant of $\partial_x \phi_n$, we find
$$
\|r'_n\| \leq (1 + \frac{L_f'}N)\|r'_{n+1}\| + \frac1N g\Delta \|x_{n+1} - \tilde{x}_{n+1}\|
$$
The last term is controled with the previous proposition, and we find
$$
\|r'_n\| \leq (1 + \frac{L_f'}N)\|r'_{n+1}\| + \frac{(e^{L_f'} - 1)C_f'g\Delta}{L_f'N^2}\Delta_\theta^N + O(\frac1{N^3}),
$$
which gives by unrolling:
$$
\|r_n'\| \leq\frac{(e^{L_f'} - 1)^2C_f'g\Delta}{L_f'^2N}\Delta_\theta^N + O(\frac1{N^2}).
$$

2) We can now control the gradients with respect to parameters. Since Heun's method involves parameters $\theta_n^N$ both for the computation of $x_{n}$ and $x_{n+1}$, the gradient formula is slightly more complicated than for the classical ResNet. It is the sum of two terms, the first one $\nabla^1_{\theta^N_n}L$ corresponding to iteration $n$ and the second one $\nabla^2_{\theta^N_n}L$ corresponding to iteration $n - 1$.

We have
$$
\nabla^1_{\theta^N_n}L = \frac1{2N} \left(\partial_{\theta}f(x_n, \theta_n^N)+ \frac1N\partial_x f(y_n, \theta_{n+1}^N)\partial_\theta f(x_n, \theta_n^N)\right)^\top\nabla_{x_n}L 
$$
and 
$$
\nabla^2_{\theta^N_n}L = \frac1{2N} \left(\partial_{\theta}f(y_{n-1}, \theta_{n-1}^N)\right)^{\top} (I + \frac1N\partial_x f(x_{n-1}, \theta_{n-1}^N))^\top \nabla_{x_{n-1}}L.
$$
The gradient $\nabla_{\theta^N_n}L$ is finally 
$$
\nabla_{\theta^N_n}L = \nabla^1_{\theta^N_n}L + \nabla^2_{\theta^N_n}L.
$$
Overall, these equations map the activations $x_{n}$ and $x_{n-1}$, and the gradients $\nabla_{x_{n-1}}L$ and $\nabla_{x_{n}}L$ to the gradient $\nabla_{\theta_n^N}$, which we rewrite as
$$
\nabla_{\theta^N_n}L = \Psi(x_n, x_{n-1}, \nabla_{x_{n}}L, \nabla_{x_{n-1}}L),
$$
where the function $\Psi$ is explicitly defined by the above equations.
With the memory-free backward pass, the gradient is rather estimated as
$$
\tilde{\nabla}_{\theta^N_n}L = \Psi(\tilde{x}_n, \tilde{x}_{n-1}, \nabla_{\tilde{x}_{n}}L, \nabla_{\tilde{x}_{n-1}}L).
$$

The function $\Psi$ is Lispchitz-continuous since all functions involved in its composition are Lipschitz-continuous and the activations belong to a compact set, and its Lipschitz constant scales as $\frac1N$. We write its Lipschitz constant as  $\frac{L_{\Psi}}N$, and we get:
\begin{align}
  \|\nabla_{\theta^N_n}L - \tilde{\nabla}_{\theta^N_n}L\| &= \|\Psi(x_n, x_{n-1}, \nabla_{x_{n}}L, \nabla_{x_{n-1}}L) - \Psi(\tilde{x}_n, \tilde{x}_{n-1}, \nabla_{\tilde{x}_{n}}L, \nabla_{\tilde{x}_{n-1}}L)\|\\
  &\leq \frac{L_{\Psi}}N( \|x_n - \tilde{x}_n\| + \|x_{n-1} - \tilde{x}_{n-1}\| + \|\nabla_{x_n}L - \nabla_{\tilde{x}_n}L\| + \|\nabla_{x_{n-1}}L - \nabla_{\tilde{x}_{n-1}}L\|).
\end{align}
Using the previous propositions, we get:
$$
\|\nabla_{\theta^N_n}L - \tilde{\nabla}_{\theta^N_n}L\| = O(\frac{\Delta_\theta^N}{N^2} + \frac1{N^3}).
$$
\end{proof}

\section{Experimental details}\label{app:exp_details}

In all our experiments, we use Nvidia Tesla V100 GPUs. 

\subsection{CIFAR}

For our experiments on CIFAR-10 (training from scratch), we used a batch-size of $128$ and we employed SGD with a momentum of $0.9$. The training was done over $200$ epochs. The initial learning rate was $0.1$ and we used a cosine learning rate scheduler. A constant weight decay was set to $5 \times 10^{-4}$. Standard inputs preprocessing as proposed in Pytorch \citep{paszke2017automatic} was performed.  

For our finetuning experiment on CIFAR-10, we used a batch-size of $128$ and we employed SGD with a momentum of $0.9$. The training was done over $5$ epochs. The learning rate was kept constant to $10^{-3}$. A constant weight decay was set to $5 \times 10^{-4}$. Standard inputs preprocessing as proposed in Pytorch was also performed.  

For our experiment with our simple ResNet model that processes the input by a $5 \times 5$ convolution with $16$ out channels, we used a batch-size of $256$ and we employed SGD with a momentum of $0.9$. The training was done over $90$ epochs. The learning rate was set to $10^{-1}$ and was decayed by a factor $10$ every $30$ epochs. A constant weight decay was set to $5 \times 10^{-4}$. Standard inputs preprocessing as proposed in Pytorch was also performed.  

\subsection{ImageNet}

For our experiments on ImageNet (training from scratch), we used a batch-size of $256$ and we employed SGD with a momentum of $0.9$. The training was done over $100$ epochs. The initial learning rate was $0.1$ and was decayed by a factor $10$ every $30$ epochs. A constant weight decay was set to $10^{-4}$. Standard inputs preprocessing as proposed in Pytorch was performed: normalization, random croping of size $224 \times 224$ pixels, random horizontal flip. 

For our finetuning experiment on ImageNet, we used a batch-size of $256$ and we employed SGD with a momentum of $0.9$. The training was done over $3$ epochs. The learning rate was kept constant to $5 \times 10^{-4}$. A constant weight decay was set to $10^{-4}$. Standard inputs preprocessing as proposed in Pytorch was performed: normalization, random croping of size $224 \times 224$ pixels, random horizontal flip.

\section{Architecture details}\label{app:resnet_recap}
In computer vision, the ResNet as presented in \citep{he2016deep} first applies non residual transformations to the input image: a feature extension convolution that goes to $3$ channels to 64, a batch norm, a non-linearity (ReLU) and optionally a maxpooling. 

It is then made of $4$ layers (each layer is a series of residual blocks) of various depth, all of which perform residual connections. Each of the $4$ layers works at different scales (with an input with a different number of channels): typically $64$, $128$, $256$ and $512$ respectively.
There are two types of residual blocks: Basic Blocks and Bottlenecks. Both are made of a successions of convolutions $\mathrm{conv}$, batch normalizations $\mathrm{bn}$ \citep{ioffe2015batch} and ReLU non-linearity $\sigma$.
For example, a Basic Block iterates (in a pre-activation \citep{he2016identity} fashion):
$$x \to x + \mathrm{bn}(\mathrm{conv}(\sigma(\mathrm{bn}(\mathrm{conv}(\sigma(x)))))).$$

Finally, there is a classification module: average pooling followed by a fully connected layer.

\newpage

\end{document}